\documentclass{article}

\usepackage[main, final]{neurips_2026}
\usepackage{hyphenat}

\usepackage{amsmath, amsthm, amssymb}
\usepackage[utf8]{inputenc} 
\usepackage[T1]{fontenc}    
\usepackage{url}            
\usepackage{booktabs}       
\usepackage{amsfonts}       
\usepackage{nicefrac}       
\usepackage{microtype}      
\usepackage{xcolor}         
\usepackage{thm-restate}
\usepackage{hyperref}       

\usepackage{mathtools}
\usepackage{booktabs,multirow}
\usepackage{xspace}
\usepackage[dvipsnames]{xcolor}
\usepackage{graphicx}
\usepackage{enumitem}
\usepackage{subcaption}
\usepackage{natbib}
\usepackage[algo2e,ruled,vlined,linesnumbered]{algorithm2e} 
\SetAlgoSkip{}

\usepackage{tikz}
\usetikzlibrary{shapes,arrows}
\usetikzlibrary{decorations}
\usetikzlibrary{plotmarks}
\usetikzlibrary{calc}
\usetikzlibrary{mindmap}
\usetikzlibrary{shadows}
\usetikzlibrary{backgrounds}
\usetikzlibrary{patterns}
\usetikzlibrary{shapes.symbols}

\allowdisplaybreaks

\newtheorem{theorem}             {Theorem}
\newtheorem{lemma}      [theorem]{Lemma}

\newtheorem{definition} [theorem]{Definition}

\newcommand{\CLIMB}{\textsc{CLIMB}_\ell \xspace}

\newcommand{\NSGADYN}{NSGA-II-DYN\xspace}

\DeclareMathOperator{\Unif}{Unif}                         

\usepackage{xcolor}

\newcommand{\ones}[1]{|#1|_1}                             

\newcommand{\COCZ}{\textsc{COCZ}\xspace}                           

\newcommand{\OMM}{\textsc{OMM}\xspace}                             

\newcommand{\LOTZ}{\textsc{LOTZ}\xspace}                           

\newcommand{\cdist}{\ensuremath{\textsc{cDist}}\xspace}   

\newcommand{\expect}[1]{\mathrm{E}\left[#1\right]}        

\newcommand{\nsga}{NSGA\nobreakdash-II\xspace}
\newcommand{\nsgaIII}{NSGA\nobreakdash-III\xspace}

\newcommand{\zeros}[1]{|#1|_0}

\title{Provable Speedups From Dynamic Population Sizes in Evolutionary Algorithms for Multiobjective Optimization}

\author{%
  Andre Opris\\
  Chair of Algorithms for Intelligent Systems\\
  University of Passau\\
  Innstraße 43, 94032 Passau \\
  \texttt{andre.opris@uni-passau.de} \\
}

\begin{document}

\maketitle

\begin{abstract}
This paper investigates the role of dynamic population sizes in evolutionary multi-objective optimization. Although such approaches are widely used in practice, their benefits remain poorly understood, and rigorous runtime analyses explaining when and why they help are still scarce. To address this, we introduce the bi-objective problem class $\CLIMB$, where the parameter $\ell$ determines the size of the Pareto front. On this problem class, we analyze the runtime of GSEMO and the widely used \nsga. Our results show that allowing a dynamic population size for \nsga can lead to a moderate improvement, yielding a super-constant speedup for $\ell \in \omega(\log(n)) \cap o(\sqrt{n \log(n)})$ in terms of fitness evaluations. In particular, we prove that GSEMO and \NSGADYN, a version of \nsga with dynamic population sizes that we propose in this paper, can cover the Pareto front of $\CLIMB$ in expected $O(n \log(n) + \ell^2 \log(en/(n-2\ell+3)))$ evaluations, and that this bound is tight for $\ell = O(\sqrt{n \log(n)})$. In contrast, NSGA-II with a fixed population size of at least $4\ell$ requires expected $O(n\ell + n\log(n) + \ell^2 \log(en/(n-2\ell+3)))$ evaluations to cover the entire Pareto front of $\CLIMB$, while still requiring expected $\Omega(\ell n)$ evaluations to even find a Pareto-optimal search point. To the best of our knowledge, this is the first rigorous runtime analysis in bi-objective optimization that demonstrates a super-constant speedup of GSEMO over \nsga. Our analysis builds on concepts from single-objective optimization, such as the evolution of population diversity over time, and employs the well-known family-tree method to prove the lower bound.
\end{abstract}

\section{Introduction}
Many real-world optimization problems are characterized by several, often conflicting objectives. These can be tackled by evolutionary multi-objective algorithms (MOEAs) which mimic principles from natural evolution to evolve a population of solutions to find a Pareto optimal set. A human decision maker may then select a solution that best matches their needs. Due to their population-based nature, MOEAs are among the most prominent approaches to such problems and have found applications in numerous practical domains like vehicle design~\citep{MultiVehicles}, scheduling problems~\citep{704576}, or research fields like artificial intelligence~\citep{LUUKKONEN2023102537}, and especially machine learning~\citep{ZHANG2025112961,NSGAMACHINE,CHOWDHURY2026109365} and neural networks~\citep{NSGANEURAL,MARTINEZCOMESANA2023106770}. The most prominent MOEA is the \nsga~\citep{Deb2002} (with around 65,000 citations). In each iteration (or generation), \nsga forms a combined parent–offspring population, which is then partitioned into layers of non-dominated fitness values for survival selection. The first layer contains only non-dominated individuals, while the $i$-th layer consists only of those individuals which are dominated by an individual with rank in $\{1, \ldots , i-1\}$. Individuals are first selected according to their layer rank and ties are broken using crowding distance as a secondary criterion. In a nutshell, the algorithm favors solutions in less crowded regions of the objective space among individuals of equal rank. It has been shown that this algorithm optimizes bi-objective problems very efficiently (see~\citep{10.1007/978-3-540-70928-2_55} for empirical results or~\citep{ZhengLuiDoerrAAAI22,Qu2022PPSN,Dang2024} for rigorous runtime analyses). Despite these extensive studies, a complete theoretical understanding of when and why the algorithm performs well is still lacking, particularly from the perspective of runtime analysis. This is somewhat surprising, given that the mathematical analysis of MOEAs began more than 20 years ago with results on the (global) simple evolutionary multiobjective optimizer ((G)SEMO)~\citep{LaumannsTZWD02,Giel2003,Thierens03}. GSEMO is particularly simple, as it relies solely on Pareto dominance for selection and generates only a single offspring per generation. Due to its simplicity, it was the first MOEA for which rigorous runtime analyses were carried out, and it remains a central algorithm in which new phenomena are first observed (see, e.g., \citep{DinotDHW23,DaOp2023,doerr2025tightruntimeGSEMO}, as well as applications to submodular optimization~\citep{QianYZ15nips,QianZTY18}). Building on these insights, similar analyses have recently been extended to other widely used MOEAs, including NSGA-II, NSGA-III, SMS-EMOA, and SPEA2 (see~\citep{ZhengLuiDoerrAAAI22,WiethegerD23,OprisNSGAIII,Opris2025,Zheng_Doerr_2024,RenBLQ24,ijcai2026p705} for breakthroughs and recent analyses). A key difference between GSEMO and these algorithms is that the latter use a fixed population size. In contrast, a dynamic population size has the potential to avoid unnecessary fitness evaluations on low-quality individuals and thereby accelerate the optimization process. While empirical studies already indicate advantages of dynamic population sizes in evolutionary multi-objective optimization~\citep{WANG2022101104,TIAN2025129296,Liang2023}, to the best of our knowledge, there is no corresponding result from the perspective of runtime analysis apart from~\citep{ijcai2026p698}. In particular, we are not aware of any benchmark problem on which GSEMO provably outperforms these popular MOEAs. Instead, most existing results suggest that, in terms of fitness evaluations, GSEMO and these algorithms achieve the same runtime guarantees~\citep{Laumanns2004,ZhengLuiDoerrAAAI22}. These insights may guide practitioners in designing improved variants of NSGA-II, NSGA-III, SMS-EMOA, and SPEA2 with enhanced performance, particularly for complex problems featuring rugged fitness landscapes.


\textbf{Our contribution:} We propose a pseudo-Boolean function class $\CLIMB$ that serves as an example where using dynamic population sizes in \nsga leads to a super-constant speedup for some values of $\ell$. Here, $\ell$ is the size of the Pareto front. In essence, reaching the Pareto front of $\CLIMB$ for small values of $\ell$ is more difficult than covering it. With a static population size, \nsga requires at least $\Omega(\ell n)$ fitness evaluations just to find a single Pareto optimal point. In this setting, an upper bound on the number of fitness evaluations for \nsga for covering the front is $O(n\ell + n\log(n) + \ell^2 \log(en/(n-2\ell+3)))$. We also prove that this upper runtime bound for \nsga is tight for $\ell=O(\sqrt{n\log(n)})$. In contrast, GSEMO and \NSGADYN, the latter being a variant of \nsga with a dynamic population size which we will introduce in this paper (see Algorithm~1 below), can cover the Pareto front of $\CLIMB$ in an expected $O(n\log(n) + \ell^2 \log(en/(n-2\ell+3)))$ number of evaluations. Hence, those algorithms achieve a super-constant speedup in runtime for $\ell \in \omega(\log(n)) \cap o(\sqrt{n \log(n)})$. For example, if $\ell=\sqrt{n}$, the speedup is of order $\Omega(\sqrt{n}/\log(n))$. For \NSGADYN, we adapt the population size $\mu_{t+1} = \min \{\mu_t + \lambda_t, 4|S_t|\}$ in each generation, where $S_t$ is a maximum set of mutually incomparable non-dominated solutions of $R_t$ in generation $t$ and $\lambda_t$ is the offspring population size (see Section~2 for details). For the upper bound, a challenge is to bound the maximum population size in expectation during the hill climbing process, depending on the progress of the algorithm already made towards the Pareto front. One also needs arguments about the population dynamics in the rare case of an unsuccessful initialization. To this end, we use the total sum of Hamming distances of the entire population from given search points, with respect to a specified set of genes, to demonstrate a positive drift towards promising regions of the search space. This is used in~\citep{ClearingI} to describe the \emph{clearing} mechanism to preserve population diversity (see also~\citep{Diversity} for similar results). For the lower bound, we apply the well-known family tree argument from~\citep{WittLower}. For our purposes, we also adapt the general arguments of~\citep{ZhengLuiDoerrAAAI22} showing that \nsga preserves high-quality solutions to this setting, including the case of dynamic population sizes. We are also confident that we obtain very similar results when introducing and analyzing dynamic-population variants of NSGA-III, SMS-EMOA, or SPEA2.

\textbf{Related work:} The theoretical runtime analysis of MOEAs initially focused on relatively simple algorithms such as (G)SEMO~\citep{LaumannsTZWD02,Laumanns2004,Giel2003,Thierens03}, and was later extended to \nsga, the most widely used MOEA, for the problem to cover the Pareto front~\citep{ZhengLuiDoerrAAAI22,Qu2022PPSN,DoerrQ23b,DangEfficient2024,Dang2024,UpBian}, including combinatorial optimization problems~\citep{Cerf2023,MOEASubset}. Runtime analyses of other popular MOEAs on simple benchmark functions have only emerged in recent years. These include SMS-EMOA~\citep{Zheng_Doerr_2024,ijcai2025p988}, SPEA2~\citep{RenBLQ24,DoerrKS2026}, variants of \nsga~\citep{Krejca2025b}, and NSGA-III~\citep{WiethegerD23,OprisNSGAIII,DoerrNearTight,OPRIS2026,OprisMultimodal,Opris26PopDyn}. It has further been shown that these algorithms can outperform GSEMO exponentially~\citep{DangOprisSudholt2026Lessons}, including in terms of approximating the Pareto front~\cite{li2026why}, particularly when the Pareto front grows exponentially in the problem size $n$. Even for simpler problems, there are results on approximating the Pareto front when its size exceeds the population size. This includes studies on \nsga~\citep{Zheng2022}, NSGA-III~\citep{ApproximationNSGAIII}, and SPEA2~\citep{ApproximationSPEA2}, with the latter two also demonstrating advantages over \nsga. Overall, the theoretical analysis of MOEAs remains an active area of research, with even the relatively simple GSEMO algorithm being the subject of recent work~\citep{doerr2025tightruntimeGSEMO}. \citet{doerr2025tightruntimeGSEMO} established tight runtime bounds for GSEMO on the well-known COCZ benchmark, but did not analyze population growth during hill climbing towards the Pareto front. Thus, it remains open whether the population becomes linear or stays sublinear. In contrast, hill climbing on $\CLIMB$ is more aggressive: every improvement in hill climbing dominates all previous solutions, causing GSEMO's population to immediately shrink to a single point. This property enables a precise analysis of the population size of \NSGADYN during this phase. Moreover, their work does not establish lower runtime bounds for GSEMO on pseudo-Boolean multi-objective problems for reaching the Pareto front, nor does it consider NSGA-II.

However, how dynamic population sizes can accelerate multi-objective optimization remains largely open. Beyond~\citep{ijcai2026p698}, which analyzes an \nsga variant with dynamic population size on the bi-objective OneMinMax and OneJumpZeroJump functions, we are not aware of runtime analyses of dynamic population sizes in multi-objective optimization. Their analysis focuses primarily on the behavior of Pareto-optimal solutions. In contrast, we analyze \NSGADYN both during hill climbing towards the Pareto front and after reaching it. Controlling the maximum population size during hill climbing is a central technical contribution. We further show that \NSGADYN can match GSEMO's runtime guarantees while maintaining a small population, whereas \nsga may exhibit worse runtime. To the best of our knowledge, this paper is the first to theoretically demonstrate such advantages of GSEMO. Further, there is a substantial difference between their \nsga variant with dynamic population size and ours. Their variant uses a fixed time window $\tau$, after which the population size is doubled, independently of the current population and, in particular, of the number of mutually incomparable rank-one solutions in the population. Thus, the appropriate choice of $\tau$ may depend on the underlying problem, motivating their additional variant that runs several phase lengths in parallel. However, this approach may incur substantial computational overhead, as multiple runs have to be maintained simultaneously. In contrast, our \NSGADYN adapts its population size based on the observed number of mutually incomparable rank-one solutions, which makes it inherently adaptive without tuning an additional parameter.

\section{Preliminaries}
\textbf{Notation}: Denote by $\mathbb{N}$ the set of natural numbers and by $\mathbb{N}_0$ the set of natural numbers including the zero. For a set $B$, denote by $|B|$ its cardinality, and by $\ln$ the logarithm to base $e$. Denote by $\vec{1}_m := (1, \ldots , 1)$ the unit vector of dimension $m$. For two random variables $Y$ and $Z$ on $\mathbb{N}_0$ we say that $Z$ \emph{stochastically dominates} $Y$ if $\Pr(Z \leq c) \leq \Pr(Y \leq c)$ for every $c \geq 0$. For $n \in \mathbb{N}$ let $[n]:=\{1, \ldots , n\}$ and we say that $\mu = O(\text{poly}(n))$ if $\mu$ does not grow asymptotically faster than a polynomial in $n$. The number of ones in a bit string $x \in \{0,1\}^n$ is denoted by $\ones{x}$ and the number of zeros by $\zeros{x}$, respectively. The \emph{Hamming distance} of two bit strings $x,y \in \{0,1\}^n$ is defined as $H(x,y) = \sum_{j=1}^n |x_j-y_j|$. For a $d$-objective function $f:\{0,1\}^n \to \mathbb{N}_0^d, x \mapsto (f_1(x), \ldots , f_d(x))$, and two search points $x, y \in \{0, 1\}^n$, $x$ \emph{weakly dominates} $y$, written as $x \succeq y$, if $f_{i}(x) \geq f_{i}(y)$ for all $i \in [d]$ and $x$ \emph{(strictly) dominates} $y$, written as $x \succ y$, if one inequality is strict. We call $x$ and $y$ \emph{incomparable} if neither $x \succeq y$ nor $y \succeq x$. Each solution $x$ not dominated by any other in $\{0,1\}^n$ is called \emph{Pareto optimal} and we call $f(x)$ \emph{non-dominated fitness value}. The set of all non-dominated fitness values is called \emph{Pareto front}.

\begin{algorithm2e}[t]
    Initialize $\mu_0 \in \mathbb{N}$ with $\mu_0 \leq r$\\
	Initialize $P_0 \sim \Unif( (\{0,1\}^n)^{\mu_0})$\\
	\For{$t:= 0 \to \infty$}{
		Initialize $Q_t:=\emptyset$\\
        Compute $\lambda_t = h_t(P_t)$\\
		\For{$i:=1 \to \lambda_t$}{
			Sample $s$ from $P_t$ uniformly at random \label{alg:nsga-ii:tournament}\\
			Create $s'$ by bitwise mutation on $s$ with rate $1/n$\\
			Update $Q_t:=Q_t \cup \{s'\}$\\
		}
		Set $R_t := P_t \cup Q_t$, and $\mu_{t+1} = g_t(R_t)$\\
		Partition $R_t$ into layers $F^1_{t+1},F^2_{t+1},\ldots $ of non-dominated solutions \label{line:non-dominated}\\
        Find $i^* \geq 1$ such that $\sum_{i=1}^{i^*-1} \lvert{F^i_{t+1}}\rvert < \mu_{t+1}$ and $\sum_{i=1}^{i^*} \lvert{F^i_{t+1}}\rvert \geq \mu_{t+1}$\\
        For each $x \in F^{i^*}_{t+1}$ compute $\cdist(x,F^{i^*}_{t+1})$ and let $Y_t$ be the set of $\mu_{t+1} - \sum_{i=1}^{i^*-1} \lvert{F^i_{t+1}}\rvert$ individuals from $F^{i^*}_{t+1}$ with highest crowding distance with respect to $F^{i^*}_{t+1}$ where ties are broken uniformly at random\\
		Create the next population $P_{t+1} := \bigcup_{i=1}^{i-1} F_{t+1}^{i^*} \cup Y_t$\\
	}
	\caption{\NSGADYN algorithm for maximizing a given $d$-objective function $f\colon \{0, 1\}^n \to \mathbb{R}^d$. The population size is updated by a function $g_t: \bigcup_{\ell=1}^r (\{0,1\}^n)^\ell \to \mathbb{N}$, and the number of offspring created by a function $h_t: \bigcup_{\ell=1}^r (\{0,1\}^n)^\ell \to \mathbb{N}$ in every generation $t$, where $r>0$ is a positive threshold.}
	\label{alg:nsga-ii}
\end{algorithm2e}

\textbf{Algorithms}: \nsga with dynamic population sizes, called \NSGADYN, is summarized in Algorithm~\ref{alg:nsga-ii} for bitwise mutation. Its vanilla version, where $g_t= h_t = \mu_0 = \mu$ for a given population size $\mu \in \mathbb{N}$ and all generations $t$, was introduced in~\citep{Deb2002,NSGAIICode2011}. The initial population size is set to $\mu_0 \in \mathbb{N}$, and $\mu_0$ individuals are sampled uniformly at random from $\{0,1\}^n$. In each generation, the new offspring population size $\lambda_t=h_t(P_t)$ is computed, and a population $Q_t$ of $\lambda_t$ offspring is created by repeatedly selecting a parent $s$ from $P_t$ uniformly at random and applying bitwise mutation. That is, each offspring is obtained by flipping each bit of $s$ independently with probability $1/n$. Then the joint population $R_t=P_t \cup Q_t$ of size $\mu_t + \lambda_t$ is computed, and the population size $\mu_{t+1}$ for the next generation is adapted via the function $g_t$ where $1 \leq g_t(R_t) \leq \mu_t + \lambda_t$. The functions $g_t$ and $h_t$ are deterministic and chosen by the user. To avoid unnecessarily large population sizes, we assume that $g_t$ and $h_t$ are bounded from above by a positive threshold $r$, which is also set by the user. To ensure that the Pareto front of a given function $f$ can be covered, we require $r \geq 4|S|$, where $S$ is a maximum set of mutually incomparable solutions with respect to $f$. This condition guarantees that first ranked solutions are protected between generations (see Lemma~\ref{lem:first-ranked} below). Note that \NSGADYN may still cover the Pareto front of $f$ even if $\mu_t < 4|S|$ throughout the entire run, since the size of the Pareto front can be much smaller than $|S|$, as it is the case for $\CLIMB$ for $\ell \in o(n)$ (see Definition~\ref{def:CLIMB} below). The vanilla NSGA-II from~\citep{Deb2002} works with $r=\mu$. During survival selection, the parent and offspring populations $P_t$ and $Q_t$ are joined into $R_t$, and then partitioned into layers $F^1_{t+1},F^2_{t+1},\dots$ by the \emph{non-dominated sorting algorithm} \citep{Deb2002}. The layer $F^1_{t+1}$ consists of all non-dominated search points, and $F^i_{t+1}$ for $i>1$ only contains points that are dominated by those from $F^1_{t+1},\dots,F^{i-1}_{t+1}$. Then the critical rank $i^*$ with $\sum_{i=1}^{i^*-1} \lvert{F^i_{t+1}}\rvert < \mu_{t+1}$ and $\sum_{i=1}^{i^*} \lvert{F^i_{t+1}}\rvert \geq \mu_{t+1}$ is determined (i.e. there are fewer than $\mu_{t+1}$ search points in $R_t$ with a lower rank than $i^*$, but at least $\mu_{t+1}$ search points with rank at most $i^*$). For $P_{t+1}$, all individuals with a rank lower than $i^*$ are selected. The remaining individuals are all taken from $F^{i^*}_{t+1}$, particularly those with highest crowding distance with respect to $M:=F^{i^*}_{t+1}$ where ties are broken uniformly at random. For $M=(x_1,x_2,\dots,x_{|M|})$ the crowding distance is computed as follows. First, sort the elements of $M$ as $M=(x_{k_1},\dots,x_{k_{\vert{M}\vert}})$
with respect to each objective $k \in [d]$
separately in descending order.
\begin{align}
	\cdist(x_i, M)
	&:= \sum_{k=1}^{d} \cdist_{k}(x_i, M), \label{eq:crowd-dist}
	\text{ where }\\
	\cdist_{k}(x_{k_i}, M)
	&\!:= \! \begin{cases}
		\infty\; & \text{if } i \in \{1, |M|\},\\
        0\; & \text{if } i \notin \{1, |M|\} \text{ and } f_k(x_{k_1}) = f_k(x_{k_M}),\\
		\frac{f_k\left(x_{k_{i-1}}\right) - f_k\left(x_{k_{i+1}}\right)}{f_k\left(x_{k_1}\right) - f_k\left(x_{k_{M}}\right)} & \text{otherwise.}
	\end{cases}\!\!\!\!\label{eq:crowd-dist-eachdim}
\end{align}

Then the first- and last-ranked individuals are always assigned an infinite crowding distance with respect to objective $k$. If the first- and last-ranked individuals have identical values in objective $k$, all remaining individuals are assigned a crowding distance of zero with respect to objective $k$. Otherwise, the crowding distance is computed as the difference between the $f_k$ values of the neighboring individuals ordered immediately above and below. This value is then normalized by the difference between the $f_k$ values of the first and last ordered individuals. We propose a $(\mu+\lambda)$ version of \NSGADYN to capture steady state variants on the one hand, when only a few offspring are created per generation (see for example~\citep{NSGASTEADYSTATE}), and the vanilla version on the other hand, when $\mu=\lambda$ remains constant over time.

\begin{algorithm2e}[t]
	Initialize $P_0:=\{s\}$ where $s \sim \Unif(\{0,1\}^n)$\\
	\For{$t:= 0 \to \infty$}{
		Sample $s \sim \Unif(P_t)$ \\
		Create $s'$ by bitwise mutation on $s$ with rate $1/n$\\
		\If{$s'$ is \textbf{not} dominated by any individual in $P_t$}{
			Create the next population $P_{t+1} := P_t \cup \{s'\}$ \\
			Remove all $x \in P_{t+1}$ weakly dominated by $s'$\\
		}
        \Else{$P_{t+1}=P_t$}
	}
	\caption{The GSEMO algorithm~{\protect\citep{LaumannsTZWD02,Giel2003}} for maximizing a given bi-objective function $f\colon \{0, 1\}^n \to \mathbb{R}^d$.}
	\label{alg:gsemo}
\end{algorithm2e}

The GSEMO algorithm is shown in Algorithm~\ref{alg:gsemo}.
Starting from one randomly generated solution, in each generation
a new search point $s'$ is created by bitwise mutation with rate $1/n$ on a parent selected uniformly at random from $P_t$. If $s'$ is not dominated by any solutions of the current population $P_t$
then it is added to $P_t$, and those weakly dominated by $s'$
are removed from $P_t$. Note that the population of GSEMO contains only non-dominated solutions with different fitness vectors, and the population size may vary.

\textbf{Runtime}: We measure runtime as the number of fitness evaluations needed to cover the entire Pareto front, rather than generations, since in \NSGADYN the number of evaluations per generation varies due to its dynamic offspring population size.

\textbf{Problem class}: In this paper, we introduce the bi-objective $\CLIMB$ function class to illustrate the benefits of using a dynamic population size in NSGA-II, where the parameter $\ell$ determines the size of the Pareto front. The bit string is split into two halves of length $n/2$. In the first half, both objectives align and simply count ones, creating a strong hill climbing signal: solutions with more ones dominate those with fewer, provided that they have a fitness distinct from zero. In the second half, we simultaneously maximize ones and zeros in an equal way, making the objectives conflicting. As the number of ones in the first half increases, the number of mutually incomparable solutions decreases moderately until it reaches $\ell$, the size of the Pareto front. This property ensures that hill climbing remains easy for both GSEMO and \NSGADYN. However, for a broad range of $\ell$, reaching the Pareto front via hill climbing is intended to be more difficult than covering it once a Pareto-optimal solution has been found.

\begin{definition}
\label{def:CLIMB}
Let $n$ be a natural number divisible by $2$ and let $b(n) \in \{1,2\}$ with $b(n)=1$ if $n$ is divisible by four, and $b(n)=2$ otherwise. Then for $\ell \in \{b(n),b(n)+2, \ldots , n/2-1, n/2+1\}$ the bi-objective function $\CLIMB: \{0,1\}^n \to \mathbb{N}_0^2$ is defined as $\CLIMB(x) = (f_1(x),f_2(x))$ with $(f_1(x),f_2(x))=$ 
\[
\begin{cases}
    n \cdot \ones{x^1} \vec{1}_2 + (\ones{x^2},\zeros{x^2}), & \text{ if $(\frac{1}{2}-\frac{\ell-1}{n})\ones{x^1} \leq \ones{x^2} \leq \frac{n}{2}-(\frac{1}{2}-\frac{\ell-1}{n})\ones{x^1}$,} \\
    (0,0), & \text{ else,}
\end{cases}
\]
for all $x=(x_1, \ldots ,x_n) \in \{0,1\}^n$, where $x^1:=(x_1, \ldots , x_{n/2})$ denotes the first half, and $x^2:=(x_{n/2+1}, \ldots , x_n)$ the second half of $x$. 
\end{definition}

To simplify the notation, particularly those involving rounding issues, we assume throughout the whole paper that $n$ is divisible by four implying that $\ell$ is an odd number. All the results below can be easily adapted to the general case where $n$ is not divisible by four, since this only affects constants in the analysis.

Note that all Pareto optimal search points $x$ of $\CLIMB$ satisfy $\ones{x^1} = n/2$ and $n/4 - (\ell-1)/2 \leq \ones{x^2} \leq n/4 + (\ell - 1)/2$. Hence, the Pareto front $\mathcal{F}$ of $\CLIMB$ is $$\mathcal{F} = \Bigl\{\Bigl(\frac{n^2}{2} + k,\frac{n^2}{2} + \frac{n}{2}-k \Bigr) \; \Bigr| \; k \in \Bigl\{\frac{n}{4}-\frac{\ell-1}{2}, \ldots , \frac{n}{4}+\frac{\ell-1}{2} \Bigr\}\Bigr\}$$
which has cardinality $n/4+(\ell-1)/2 - (n/4-(\ell-1)/2)+1 = \ell$. So the vanilla NSGA-II needs a population size of at least $\ell$ to optimize $\CLIMB$, particularly to cover its Pareto front, which means that for every $v \in \mathcal{F}$ there is an $x \in P_t$ with $f(x)=v$. The condition ensuring $f(x) \neq 0$ in the first line of the definition of $\CLIMB$ guarantees that, with overwhelming probability, the initial search point has nonzero fitness. The typical optimization process is that individuals perform hill climbing by increasing the number of ones in their first half, while keeping a fitness distinct from zero. Furthermore, a maximum set of mutually incomparable solutions $S$ with respect to $\CLIMB$ has size at most $n/2 + 1$ for all $\ell \in \{1,3, \ldots , n/2+1\}$.

\textbf{Assumptions on $g_t$, $h_t$ and $\mu_0$:} In our main theorem (Theorem~\ref{thm:positive-result}), we assume that $\mu_0=1$ and $\mu_{t+1} = \min\{4|S_t|, \mu_t + \lambda_t\}$, where $S_t$ denotes a maximum set of mutually incomparable non-dominated (first-ranked) solutions in the joint population $R_t$. Of course, $|S_t| \leq |S|$ where $S$ is a maximum set of mutually incomparable solutions with respect to $\CLIMB$. So we let $\mu_t$ depend only on the objective vectors attained by $P_t$, not on their genotype. We consider the choice of $\mu_{t+1}$ appropriate, since non-dominated solutions are preserved across generations or even better solutions are created (see Lemma~\ref{lem:first-ranked}), and also dominated solutions in lower layers have a chance to survive. Thus, even such poor solutions can contribute to maintaining a well-spread population, as is the case for vanilla \nsga. In our main theorem, we assume that $1 \leq \lambda_t \leq \mu_t$.

\textbf{Key structural lemma}: The following key lemma generalizes Lemmas~1 and~7 from~\citep{ZhengLuiDoerrAAAI22} to arbitrary bi-objective functions within the \NSGADYN framework and, in a nutshell, states that promising solutions cannot be lost between generations. The proof idea is the same as in~\citep{ZhengLuiDoerrAAAI22}, as for every fitness vector $v$ there are at most four individuals in $R_t$ covering $v$ with crowding distance larger zero.

\begin{lemma}
\label{lem:first-ranked}
Consider \NSGADYN optimizing a bi-objective function $f: \{0,1\}^n \to \mathbb{R}^2$ with any function $h_t$, any positive threshold $r \geq 4 |S|$ for $g_t$ and population size $\mu_{t+1} \geq \min\{4 |S_t|, \mu_t + \lambda_t\}$. Let $v$ be covered by a first ranked individual $x \in F_{t+1}^1$. Then there is $y \in P_{t+1}$ with $f(y)=v$.    
\end{lemma}

\begin{proof}
If $\mu_{t+1} = \mu_t + \lambda_t$ then $P_{t+1} = R_t$ and the lemma holds. So suppose that $\mu_{t+1} = 4|S_t|$. First, we claim that there are at most $4|S_t|$ individuals with positive crowding distance in $F_{t+1}^1$ as all individuals with positive crowding distance in $F_{t+1}^1$ survive. Let $v$ be a fitness vector covered by at least five first ranked individuals. This means that there are $m \geq 5$ individuals $x_1, \ldots , x_m$ with fitness vector $v$, sorted with respect to $f_1$ in descending order, say $(x_1, \ldots , x_m)$. With respect to $f_2$, assume the descending order is $(x_{i_1}, x_{i_2}, \ldots , x_{i_m})$ for distinct $i_1, \ldots , i_m \in [m]$. Let $j \in \{1, \ldots , m\} \setminus (\{1,m\} \cup \{i_1,i_m\})$. Then $j=i_k$ for some $k \in \{2, \ldots , m-1\}$ and $x_j$ has crowding distance zero, since $f_1(x_{j-1})-f_1(x_{j+1}) = 0$ and $f_2(x_{i_{k-1}})-f_2(x_{i_{k+1}}) = 0$, proving the claim that there are at most $4|S_t|$ individuals with positive crowding distance in $F_{t+1}^1$.

Note that for every vector $v$ covered by an individual $x \in F_{t+1}^1$, there exists an individual $y \in F_{t+1}^1$ with $f(y)=v$ that has positive crowding distance. To see this, let $V:=\{v \in \mathbb{N}_0^2 \mid \text{ there is $x \in F_{t+1}^1$ with } f(x)=v \}$. If $v_1 = \max \{ v_1 \mid v \in V \}$, then there exists an $x \in F_{t+1}^1$ with crowding distance equal to infinity (which is, for instance, $x_{k_1}$ in the above definition of the crowding distance with respect to $F_{t+1}^1$). Otherwise, consider, for example, the individual $x_i$ covering $v$ that appears first in the descending order with respect to $f_1$. Its crowding distance with respect to $F_{t+1}^1$ is at least 
$$\frac{f_1(x_{i-1}) - f_1(x_{i+1})}{f_1(x_{k_1})-f_1(x_{k_M})} > 0$$
since $f_1(x_{k_1}) > f_1(x_{k_M})$ (otherwise $v_1 = \max\{ v_1 \mid v \in V\}$). This proves the lemma, since such a $y$ survives by the claim above.
\end{proof} 

\section{Runtime Analysis of GSEMO and \NSGADYN on $\CLIMB$}

In this section, we derive upper bounds on the expected runtime of GSEMO and \NSGADYN when optimizing $\CLIMB$. Unlike vanilla NSGA-II, these algorithms allow for a variable population size. One can show that, as long as the Pareto front has not yet been reached, the population size of \NSGADYN remains small enough for hill climbing to be effective. Since GSEMO creates only one offspring per generation and discards dominated solutions, the number of fitness evaluations required to make progress towards the Pareto front is independent of the population size in our analysis. Thus, the exact population size of GSEMO does not need to be determined during the hill climbing phase. However, before hill climbing begins, an individual with nonzero fitness must be created, which occurs with overwhelming probability during initialization by restricting the number of ones in each half on a sufficiently small interval around $n/4$. The main analytical tool for this result is a Chernoff bound, which is applicable because initialization is performed uniformly at random. From Lemma~\ref{lem:initialization-success} to Lemma~\ref{lem:maximum-ones-increase}, we do not assume that $\mu_{t+1} = \min \{\mu_t + \lambda_t, 4|S_t|\}$, as these results hold in a more general setting.

\begin{lemma}
\label{lem:initialization-success}
Let $\ell \in \{1,3,\ldots ,n/2+1\}$. Consider GSEMO or \NSGADYN optimizing $\CLIMB$ with $\mu_0 = \text{poly}(n)$, any $g_t$ and any $h_t$, and any threshold $r \geq 4|S|$. Then all initialized individuals $x$ satisfy $19n/80 \leq \ones{x^1} \leq 21n/80$ and $3n/16 \leq \ones{x^2} \leq 5n/16$ with probability $1-e^{-\Omega(n)}$. In this case, $f(x) \neq 0$.
\end{lemma}

\begin{proof}
Using a standard Chernoff bound, the probability that an individual $x$ initializes with $(\ones{x^1},\ones{x^2}) \in [\frac{19n}{80},\frac{21n}{80}] \times [\frac{3n}{16},\frac{5n}{16}] \subset [\frac{n}{8},\frac{3n}{8}] \times [\frac{3n}{16},\frac{5n}{16}]$ is $1-e^{-\Omega(n)}$, since the expected number of ones in each half is $\frac{n}{4}$. Note also that each search point $x$ with $(\ones{x^1},\ones{x^2}) \in [\frac{n}{8},\frac{3n}{8}] \times [\frac{3n}{16},\frac{5n}{16}]$ satisfies 
\begin{align*}
\ones{x^1}\left(\frac{1}{2}-\frac{\ell-1}{2}\right) &\leq \frac{3n}{8}\left(\frac{1}{2}-\frac{\ell-1}{2}\right) = \frac{3n}{16}-\frac{3n(\ell-1)}{16} \leq \frac{3n}{16} \leq \ones{x^2} \leq \frac{5n}{16} = \frac{n}{2} - \frac{3n}{16}\\
&\leq \frac{n}{2} - \frac{3n}{16} + \frac{3 n}{8} \frac{\ell-1}{n} = \frac{n}{2} - \frac{3n}{8}\left(\frac{1}{2} - \frac{\ell-1}{n}\right) \leq \frac{n}{2} - \left(\frac{1}{2} - \frac{\ell-1}{n}\right)\ones{x^1}
\end{align*}
and hence, $x$ has fitness distinct from zero. The statement for the whole population of size $\mu_0$ follows by a union bound, since $\mu_0 = \text{poly}(n)$. 
\end{proof}

However, to prove the upper bound, we must also handle the rare case in which all initial individuals have fitness zero. In this setting, the population exhibits a positive drift towards the region
$$W:=\{x \in \{0,1\}^n \mid n/4-\sqrt{n} \leq \ones{x^2} \leq n/4+\sqrt{n}\}.$$
After entering $W$, one can create an individual $y$ with $\ones{y^2} = n/4$ by flipping $\sqrt{n}$ specific bits at once (see Theorem~\ref{thm:positive-result} below). Such a $y$ has nonzero fitness. In what follows, we investigate the drift towards $W$. Individuals may start on either side of $W$, so we partition them accordingly. At initialization, we call an individual $x$ \emph{left} if $\ones{x^2} < n/4 - \sqrt{n}$ and \emph{right} if $\ones{x^2} > n/4 + \sqrt{n}$. Then, in generation $t \geq t_0$, we call an individual in $R_t$ \emph{left} if it is the offspring of a left individual and we call it \emph{right} if it is an offspring of a right individual. Note that if all individuals have fitness zero in $P_0$, the population consists entirely of left and right individuals. We measure progress towards $W$ using the potential
$$\Phi(P_t):=\sum_{x \in P_t} H(x^2,x^*)=\sum_{x \in P_t}\sum_{i=1}^{n/2} |(x^2)_i-x_i^*|,$$
where $x^* = 0^{n/2}$ for left individuals and $x^* = 1^{n/2}$ for right ones. First, we analyze the expected change in this potential over one generation. Note that this analysis takes place on a completely flat fitness landscape, where there is no preference for which individuals are selected or retained. 

\begin{lemma}
\label{lem:initialization-preparation}
Consider a generation $t$ of GSEMO or \NSGADYN optimizing $\CLIMB$ with $\ell \in \{1,3,\ldots,n/2+1\}$, population size $\mu:=\mu_t$, any offspring population size $\lambda:=\lambda_t$, and $\mu_{t+1}=\mu_t=\mu$. Suppose that all individuals in $P_t$ have fitness zero. Then
\[
\expect{\Phi(P_{t+1}) \mid \Phi(P_t)} = \Phi(P_t) \left(1 - \frac{2 \lambda}{n(\mu + \lambda)} \right) + \frac{\lambda \mu}{2(\mu + \lambda)},
\]
where $\lambda=\mu=1$ in the case of GSEMO.
\end{lemma}

\begin{proof}
    Let $y$ be an individual generated by mutation on $x \in P_t$. We obtain under the condition that we select $x$ as a parent, denoted by event $B_x$,
	\begin{align*}
	\expect{H(y^2,x^*) \mid B_x} &= H(x^2,x^*) + \frac{n/2 - H(x^2,x^*)}{n} - \frac{H(x^2,x^*)}{n}\\
    &= H(x^2,x^*) -\frac{2 H(x^2,x^*)}{n} + \frac{1}{2}.
	\end{align*}
    Note that the expected number of flipped bits increasing the Hamming distance between $x^2$ and $x^*$ is $(n/2 - H(x^2,x^*))/n$ (i.e. the number of positions in the second half of $x$ where $x,x^*$ coincide divided by $n$), while the expected number of bits decreasing the Hamming distance between $x^2$ and $x^*$ is $H(x^2,x^*)/n$. Hence, if $x$ among $\mu$ parents is chosen uniformly at random, then by the law of total probability
	\[
	\expect{H(y^2,x^*)} = \frac{1}{\mu} \sum_{x \in P_t} \left(H(x^2,x^*) -\frac{2 H(x^2,x^*)}{n} + \frac{1}{2}
	\right) = \frac{\Phi(P_t)}{\mu}-\frac{2 \Phi(P_t)}{\mu n}+\frac{1}{2}
	\]
	and therefore, if $\lambda$ offspring $y_1, \ldots ,y_{\lambda}$ are generated,
	\[
	\expect{\Phi(Q_t) \mid \Phi(P_t)} = \lambda \left(\frac{\Phi(P_t)}{\mu}-\frac{2 \Phi(P_t)}{\mu n}+\frac{1}{2}\right) = \frac{\lambda \Phi(P_t)}{\mu} -\frac{2 \lambda \Phi(P_t)}{\mu n}+\frac{\lambda}{2}
	\]
	where $Q_t=\{y_1, \ldots , y_{\lambda}\}$. This implies for the joint population $R_t = P_t \cup Q_t$ that 
    \begin{align*}
    \expect{\Phi(R_t) \mid \Phi(P_t)} &= \expect{\Phi(P_t) \mid \Phi(P_t)} + \expect{\Phi(Q_t) \mid \Phi(P_t)} \\
    &= \left(1+\frac{\lambda}{\mu} \right)\Phi(P_t) -\frac{2 \lambda \Phi(P_t)}{\mu n}+\frac{\lambda}{2}.
    \end{align*}
    
    To form $P_{t+1}$, $\lambda$ individuals chosen uniformly at random are removed from $R_t$ (since $\mu_{t+1}=\mu_t=\mu$). Note that the probability to remove a certain individual $x \in R_t$ is $\lambda/(\mu + \lambda)$, and hence 
	\begin{align*}
	\expect{\Phi(P_{t+1}) \mid \Phi(R_t)} &= \Phi(R_t) - \sum_{x \in R_t} \frac{\lambda H(x^2,x^*)}{\mu + \lambda} = \Phi(R_t) - \frac{\lambda}{\mu + \lambda} \Phi(R_t) = \frac{\mu \Phi(R_t)}{\mu + \lambda}
	\end{align*}
    which implies 
    \begin{align*}
	\expect{\Phi(P_{t+1}) \mid \Phi(P_t)} &= \frac{\mu}{\mu + \lambda} \left(\left(1+\frac{\lambda}{\mu} \right)\Phi(P_t) -\frac{2 \lambda \Phi(P_t)}{\mu n}+\frac{\lambda}{2}\right)\\
    &= \Phi(P_t)-\frac{2 \lambda \Phi(P_t)}{n(\mu + \lambda)} +\frac{\lambda \mu}{2(\mu + \lambda)}\\
    &= \Phi(P_t) \left(1 - \frac{2 \lambda}{n(\mu + \lambda)} \right) + \frac{\lambda \mu}{2(\mu + \lambda)}.
	\end{align*}
    This proves the lemma by noting that GSEMO only creates one offspring in each iteration, and satisfies $|P_t|=1$. 
\end{proof}

Surprisingly, $\Phi(P_{t+1})$ depends only on $\Phi(P_t)$, the population size $\mu$, and the offspring population size $\lambda$, and not on the exact genotypes of the $\mu$ individuals in $P_t$, in particular, not on their distribution in the search space. We can interpret the change $\Phi(P_{t+1}) - \Phi(P_t)$ as genetic drift as long as no individual has entered the region $W$, and use the additive drift theorem~\citep{Jun2004} to bound the expected number of fitness evaluations until an individual enters $W$. 

\begin{lemma}
\label{lem:initialization-drift}
Consider GSEMO or \NSGADYN optimizing $\CLIMB$ for $\ell \in \{1,3,\ldots,n/2+1\}$, with population size $\mu := \mu_t$ and any offspring population size $\lambda_t$ bounded by $\lambda \in \text{poly}(n)$ for all $t \geq 1$. Further, let $\mu,\lambda \in \text{poly}(n)$. Suppose that all individuals in $P_t$ have fitness zero. Then the expected number of fitness evaluations required to generate an individual in $W$ is $O(\lambda \mu n \sqrt{n})$ for \NSGADYN and $O(n \sqrt{n})$ for GSEMO.
\end{lemma}

\begin{proof}
Define $\Phi(P_t):= \sum_{x \in P_t} H(x^2,x^*)$. Let $T$ be the time until we created an individual $x$ with $H(x^2,x^*) \geq n/4-\sqrt{n}$. We obtain $H(x^2,x^*) \geq n/4-\sqrt{n}$ for an individual $x \in P_t$ if $\Phi(P_t) \geq \mu n/4-\mu\sqrt{n}$. So suppose $\Phi(P_t) < \mu n/4-\mu\sqrt{n}$. We see by Lemma~\ref{lem:initialization-preparation} that for all $\lambda_t \geq 1$
$$\expect{\Phi(P_{t+1}) \mid \Phi(P_t)} = \Phi(P_t)\left(1 - \frac{2 \lambda_t}{n(\mu + \lambda_t)} \right) + \frac{\lambda_t \mu}{2(\mu + \lambda_t)}$$
and therefore, 
\begin{align*}
\expect{\Phi(P_{t+1})-\Phi(P_t) \mid \Phi(P_t)} &= \frac{\lambda \mu}{2(\mu + \lambda_t)} - \frac{2 \lambda_t}{n(\mu + \lambda_t)} \Phi(P_t) \\
&> \frac{\lambda_t \mu}{2(\mu + \lambda_t)} - \frac{2 \lambda_t \mu}{n(\mu + \lambda_t)} (n/4-\sqrt{n}) \\
&= \frac{2 \lambda_t \mu}{\sqrt{n}(\mu + \lambda_t)} \geq \frac{1}{\sqrt{n}} =:\delta.
\end{align*}
By the additive drift theorem~\citep{Jun2004} we obtain $\expect{T} \leq (\mu n/4-\mu\sqrt{n})/\delta \leq \mu n \sqrt{n}/4$. In other words, after at most $\mu n \sqrt{n}/4$ generations in expectation, we created a search point $x$ with $H_2(x^*,x) \geq n/4-\sqrt{n}$. In this case, either this search point is in $W$, or is an offspring of an individual $z$ with the following property. If $\ones{z^2}<n/4-\sqrt{n}$ then $\ones{x^2}>n/4+\sqrt{n}$, or if $\zeros{z^2}<n/4-\sqrt{n}$ then $\zeros{x^2}>n/4+\sqrt{n}$. In other words, a mutation on $z$ causes $\ones{z^2}$ to cross the interval $[n/4-\sqrt{n},n/4+\sqrt{n}]$ which we denote as a \emph{failure}. However, this requires to flip at least $\lceil{2 \sqrt{n}}\rceil$ many bits at once. So the probability of a failure is at most
\begin{align*}
\frac{\binom{n/2}{\lceil{2 \sqrt{n}}\rceil}}{n^{\lceil{2 \sqrt{n}}\rceil}} \leq \frac{(n/2)^{\lceil{2 \sqrt{n}}\rceil} \cdot e^{\lceil{2 \sqrt{n}}\rceil} / (\lceil{2\sqrt{n}}\rceil)^{\lceil{2 \sqrt{n}}\rceil}}{n^{\lceil{2 \sqrt{n}}\rceil}} = O\left(\frac{1}{\left(\sqrt{n}\right)^{2\sqrt{n}}}\right) = O\left(\frac{1}{n^{\sqrt{n}}}\right)
\end{align*}
where we used the standard inequality $\binom{n}{k} \leq \left(\frac{ne}{k}\right)^k$ for all $k,n \in \mathbb{N}$ with $k \leq n$. By a union bound, the probability that a failure occurs in some given generation $t \le T$ is at most $O(\lambda^*/n^{\sqrt{n}})$ where $\lambda^*:=\max\{\lambda_t \mid t \leq T\} \leq \lambda = O(\text{poly}(n))$. Since $\expect{T} = O(\mu n \sqrt{n})$, again by a union bound, a failure occurs with probability $O(\lambda^* \mu n \sqrt{n}/n^{\sqrt{n}}) = o(1)$ within $T$ generations. If a failure occurs in some generation $t \leq T$, and there is still no search point $x \in P_{t+1} \cap W$, we repeat the above arguments in a new period of generations of length $T$, where we redefine "left" and "right" individuals with respect to generation $T+1$ as follows: We call an individual $y$ \emph{left} if $\ones{y^2} < n/4-\sqrt{n}$ in generation $T+1$ and otherwise \emph{right}. Then, in generation $t'>T+1$, we call an individual \emph{left}, if it is the offspring of a left individual and otherwise, if it is the offspring of a right individual, call it \emph{right}. The expected number of periods is $1 + o(1)$, which proves the lemma, since the expected number of fitness evaluations to enter $W$ is at most $O(\mu \lambda n \sqrt{n})$, since $\lambda_t \leq \lambda$ for all $t \geq 0$.
\end{proof}

After initialization, and after possibly adjusting the population to ensure that at least one individual has nonzero fitness, we aim to reach a Pareto optimal solution via hill climbing. To achieve this, we exploit the strong fitness signal provided by the number of ones in the first half of the bit string, gradually increasing this value while avoiding the creation of individuals with zero fitness. In the following lemma, we derive a lower bound on the probability of increasing, through mutation, the maximum number of ones in the first half without generating an individual with zero fitness.

\begin{lemma}{onesincrease}
\label{lem:maximum-ones-increase}
Consider a generation $t$ of GSEMO or \NSGADYN for any choice of $g$ and $h$ when optimizing $\CLIMB$ for $\ell \in \{1,3, \ldots , n/2+1\}$, and any threshold $r \geq 1$. Suppose that $\delta_t:=\max\{i \in \{0, \ldots , n/2\} \mid \text{there is $x \in P_t$ with $\ones{x^1}=i$ and $f(x) \neq 0$}\} \in \{0, \ldots , n/2-1\}$ which means that there exists $x \in P_t$ with nonzero fitness, but no Pareto optimal search point has been found yet. If an individual $x \in P_t$ with $\ones{x^1} = \delta_t$ is chosen as a parent, then $x$ mutates to a $y$ with $\ones{y^1}>\ones{x^1}$ and $f(y) \neq 0$ with probability at least $(n/2-\delta_t)/(4en)$. This increases $\delta_t$. 
\end{lemma}

\begin{proof}
One can create a $z$ with $\ones{z^2} > \delta_t$ by flipping a zero in the first half of $x$, and flipping a one in the second half of $x$ if $\ones{x^2}>\frac{n}{4}$, a zero there if $\ones{x^2} < \frac{n}{4}$, and no bit there if $\ones{x^2} = \frac{n}{4}$. This happens with probability at least $$\frac{n/2-\delta_t}{n} \cdot \frac{n/4}{n} \cdot \left(1-\frac{1}{n}\right)^{n-2} \geq \frac{n/2-\delta_t}{4en}.$$
We still have to argue why $y:=z$ has fitness distinct from zero or, in other words, why $(\frac{1}{2}-\frac{\ell-1}{n})\ones{y^1} \leq \ones{y^2} \leq \frac{n}{2}-(\frac{1}{2}-\frac{\ell-1}{n})\ones{y^1}$ holds. Note that $(\frac{1}{2}-\frac{\ell-1}{n})\ones{x^1} \leq \ones{x^2} \leq \frac{n}{2}-(\frac{1}{2}-\frac{\ell-1}{n})\ones{x^1}$. So if $\ones{x^2} < \frac{n}{4}$ then $\ones{y^2} = \ones{x^2}+1$, and we obtain
\begin{align*}
\left(\frac{1}{2}-\frac{\ell-1}{n}\right)\ones{y^1} &= \left(\frac{1}{2}-\frac{\ell-1}{n}\right)(\ones{x^1}+1) = \left(\frac{1}{2}-\frac{\ell-1}{n}\right)\ones{x^1}+\frac{1}{2}-\frac{\ell-1}{n}\\
&\leq \left(\frac{1}{2}-\frac{\ell-1}{n}\right)\ones{x^1}+1 \leq \ones{x^2}+1 = \ones{y^2},
\end{align*}
as well as
\begin{align*}
\ones{y^2} \leq \frac{n}{4} \leq \frac{n}{2} - \frac{n}{4} +\frac{\ell-1}{2} = \frac{n}{2} - \left(\frac{1}{2} - \frac{\ell-1}{n}\right)\frac{n}{2} \leq \frac{n}{2} - \left(\frac{1}{2}-\frac{\ell-1}{n}\right)\ones{y^1}, 
\end{align*}
since $\ones{y^1}$ is naturally bounded by $n/2$ from above. In a similar way, if $\ones{x^2} > \frac{n}{4}$ then $\ones{y^2} = \ones{x^2}-1 \geq \frac{n}{4}$ and therefore 
\begin{align*}
\left(\frac{1}{2}-\frac{\ell-1}{n}\right)\ones{y^1} \leq \frac{n}{4} -\frac{\ell-1}{2} \leq \frac{n}{4} \leq \ones{y^2}. 
\end{align*}
Further, since $\ones{x^2}$ is naturally bounded by $n/2$ from above,
\begin{align*}
\ones{y^2} &= \ones{x^2}-1 \leq \frac{n}{2}-\left(\frac{1}{2}-\frac{\ell-1}{n}\right)\frac{n}{2} - 1 \leq \frac{n}{2}-\left(\frac{1}{2}-\frac{\ell-1}{n}\right)\ones{x^1} - 1\\
&\leq \frac{n}{2}-\left(\frac{1}{2}-\frac{\ell-1}{n}\right)(\ones{x^1}+1) = \frac{n}{2}-\left(\frac{1}{2}-\frac{\ell-1}{n}\right)\ones{y^1},
\end{align*}
proving the lemma, since by Lemma~\ref{lem:first-ranked}, an individual $y$ with increased $\delta_t$ survives.
\end{proof}

However, GSEMO and \NSGADYN behave quite differently during the hill climbing phase. In GSEMO, the population contains only non-dominated solutions, so some parent with $\ones{x^1} = \delta_t$ is always selected, and in each generation a single offspring is created. In contrast, \NSGADYN generates $\lambda_t$ offspring per generation and may also retain dominated solutions. This can result in $\Omega(n)$ non-dominated individuals among the first-ranked solutions, and thus a population of the same order. Such growth slows down hill climbing, since a single generation then already requires $\Omega(n)$ fitness evaluations. Controlling the population size is therefore crucial in the analysis of \NSGADYN on $\CLIMB$, as shown in the following lemma. Recall that $S_t$ is a maximum set of mutually incomparable, non-dominated solutions from $R_t$ with respect to $\CLIMB$, and $S$ is a maximum set of such solutions in $\{0,1\}^n$. 

\begin{lemma}{expectedpopulationsize}
\label{lem:expected-population-bounded-NSGADYN}
Consider \NSGADYN optimizing $\CLIMB$ for $\ell \in \{1,3, \ldots , n/2+1\}$ with population sizes $\mu_0=1$ and $\mu_{t+1}=\min\{\mu_t+\lambda_t,4|S_t|\}$, where $1 \leq \lambda_t \leq \mu_t$ is the offspring population size, and any threshold $r \geq 4|S|$. For $i \in \{0, \ldots , n/2-1\}$ define $p_i := (n/2-i)/(4en)$. Suppose that all individuals $x \in P_0$ have nonzero fitness and satisfy $\ones{x^1} \geq 19n/80$. Let $\mu_t^{(i)}$ denote the maximum population size occurring in a generation $t$ when $\delta_t = i$ holds. Then $\expect{\mu_t^{(i)}} \leq 320/p_i$.
\end{lemma}

\begin{proof}
Note that $\mu_0=1, \mu_1=2$, $\mu_2=4$ and $\mu_\ell \geq 4$ for all $\ell \geq 3$. Note also that $\mu_{t+1} \leq 4 |S_t|$ where $S_t$ is a maximum set of mutually incomparable non-dominated solutions in $R_t$ and hence, at least $\mu_{t+1}/4 \leq |S_t|$ many individuals from $P_{t+1}$ are non-dominated (since by the non-dominated sorting procedure at least $\min\{\mu_{t+1},|S_t|\} = \min \{\mu_t+\lambda_t,4|S_t|,|S_t|\} = |S_t|$ many non-dominated individuals are taken into $P_{t+1}$). Also recall that $S_t$ is a maximum set of mutually incomparable solutions from $R_t$. Hence, for each $t \geq 0$, at least $\mu_t/4$ individuals are non-dominated. For generation $t$ define $\delta_t:=\max\{i \in [n/2] \mid \text{there is an $x \in P_t$ with $\ones{x^1}=i$ and $f(x) \neq 0$}\}$. Let $i^*:=\delta_0 \geq 19n/80$. Then, the expected number of fitness evaluations required for creating an individual $y$ with $\ones{y^1}>i^*$ is at most $4/p_{i^*}$, since at least $\mu_t/4$ individuals from $P_t$ are non-dominated and hence, satisfy $\ones{x^1} = i^*$. Note that the probability is at least $1/4$ to choose such an individual as a parent in one trial. So $\expect{\mu_t^{(i^*)}} \leq 1+4/p_{i^*} \leq 320/p_{i^*}$, where the additive term $1$ accounts for the evaluation of the initial individual. For an induction on $i \in \{i^*, \ldots , n/2-2\}$, suppose that $\expect{\mu_t^{(i)}} \leq 320/p_i$. The goal is to show that $\expect{\mu_t^{(i+1)}} \leq 320/p_{i+1}$. Denote by $\mu_t^{(i,\text{up})}$ the population size in generation $t$ after $\delta_t = i$ increased in generation $t-1$. Denote by $\lambda_{t-1}^{(i,\text{up})}$ the number of offspring created in generation $t-1$. We claim the following.

\begin{lemma}
\label{lem:bound-population-next-generation}
In expectation, at most $19\text{E}[\lambda_{t-1}^{(i,\text{up})}]/80+1$ individuals $x$ from $P_t$ have nonzero fitness and satisfy $\ones{x^2}>\delta_t$.  
\end{lemma}

\begin{proof}
We show that an individual $x \in P_t$ with $\ones{x^1} \geq 19n/80$ creates an offspring $y$ with $\ones{y^1}>\delta_t$ with probability at most $19/80$. Then the result follows, since the expected number of such offspring $y$ is at most $19\lambda_{t-1}^{(i,\text{up})}/80$, and there is at least one offspring $z$ with nonzero fitness satisfying $\ones{z^1}>\delta_t$. If $\ones{x^1} \geq 21n/80$, the probability follows by a union bound, since it requires to flip a zero bit in the first half of $x$, and there are at most $n/2-21n/80 = 19n/80$ zero bits to flip there. So suppose that $19n/80 \leq \ones{x^1} < 21n/80$. Fix $L = \lfloor{19n/80}\rfloor$ positions with ones and $L' = \lfloor{19n/80}\rfloor$ positions of zeros in the first half of $x$. Denote by $C$ the remaining positions in the first half. Then, to increase the number of ones there, one has to flip a bit in the region $C$, or to create an offspring with more ones than zeros in the region $L \cup L'$. Suppose the latter happens with probability $p^{\text{up}}$. Note that this probability is the same as the probability of creating an offspring with more zeros than ones in $L \cup L'$ by symmetry. Flipping no bits in $L \cup L'$ happens with probability $(1-1/n)^{2|L|} \geq (1-1/n)^{n/2} \geq 1/\sqrt{e} - o(1)$. Since we have $2p^{\text{up}} + (1-1/n)^{2|L|} = 1$, we can estimate $p^{\text{up}} \leq 1/2-(1-1/n)^{2|L|}/2 \leq 1/2-1/(2\sqrt{e}) + o(1) \leq 1/5$ for $n$ sufficiently large. Since $|C| \leq n/40 + 2$, we see by a union bound that a bit in region $C$ is flipped with probability at most $1/40+o(1)$. Hence, we obtain for the probability to increase $\ones{x^1}$ the upper bound $1/2-1/(2\sqrt{e}) + 1/40 + o(1) < 1/5+1/40 = 9/40 < 19/80$ for $n$ sufficiently large, proving the lemma. 
\end{proof} 

With Lemma~\ref{lem:bound-population-next-generation}, we see that $\expect{\mu_t^{(i,\text{up})}} \leq 4 \cdot 19 E[\lambda_{t-1}^{(i,\text{up})}]/80 + 4 = 76 \lambda_{t-1}^{(i,\text{up})}/80+4$ (since $\mu_t^{(i,\text{up})}$ is bounded by $4|S_t|$ and $\expect{|S_t|} \leq \expect{|F_{t+1}^1|} \leq 19 \expect{\lambda_{t-1}^{(i,\text{up})}}/80+1$ since $S_t$ is the maximum number of mutually incomparable non-dominated solutions in $R_t$ which is a subset of $F_{t+1}^1$). Considering also at most $4/p_{i+1}$ additional offspring creations until we increase $\delta_t=i+1$ again, we see that
\begin{align*}
\expect{\mu_t^{(i+1)}} &\leq \expect{\mu_t^{(i,\text{up})}} + \frac{4}{p_{i+1}} \leq \frac{76 \expect{\lambda_{t-1}^{(i,\text{up})}}}{80} + 4 + \frac{4}{p_{i+1}}\\
&\leq 76 \cdot \frac{320}{80p_i} + 4 + \frac{4}{p_{i+1}} \leq \frac{304}{p_{i+1}} + 4 + \frac{4}{p_{i+1}} = \frac{308}{p_{i+1}} + 4 \leq \frac{320}{p_{i+1}},
\end{align*}
proving the lemma.
\end{proof}

Once the Pareto front is reached, the runtime to cover it can be easily computed using standard arguments from~\citep{ZhengLuiDoerrAAAI22,Opris2025}. 

\begin{lemma}
\label{lem:covering-front}
Consider GSEMO or \NSGADYN optimizing $\CLIMB$ for $\ell \in \{1,3, \ldots , n/2+1\}$ and suppose there is a generation $t_0$ such that $P_{t_0}$ contains at least one Pareto optimal solution. Then the following statements hold.
\begin{enumerate}
    \item[(1)] GSEMO needs $O(\ell^2 \ln(en/(n-2\ell+3)))$ fitness evaluations in expectation to cover the Pareto front of $\CLIMB$. 
    \item[(2)] For \NSGADYN, suppose that there are $\mu,\lambda \in \mathbb{N}$ with $1 \leq \mu_t \leq \mu$ and $1 \leq \lambda_t \leq \lambda$ for all $t \geq t_0$. Further, suppose that $\mu_{t+1} \geq \min\{\mu_t+\lambda_t,4|S_t|\}$ for all $t \geq 0$. Then \NSGADYN needs $O(\lambda \ell + \mu\ell\ln(en/(n-2\ell+3)))$ fitness evaluations in expectation to cover the Pareto front of $\CLIMB$.   
\end{enumerate}
\end{lemma}

\begin{proof}
Without loss of generality, we can assume that $t_0=0$. For $t \geq 0$, let $\mathcal{F}_t$ denote the set of all Pareto optimal individuals in the current population $P_t$. Note that $\mathcal{F}_t \neq \emptyset$. As long as the Pareto front is not yet fully covered, there always exists an uncovered fitness vector $v$ on the Pareto front that has a covered neighbor $w$. That is, $|v_1-w_1|=1$ and there exists $x\in P_t$ such that $f(x)=w=(n^2/2+\ones{x^2},n^2/2+n/2-\ones{x^2})$, while there is no $y \in P_t$ such that $f(y)=v$. Call such a $v$ an \emph{adjoining} fitness vector. All solutions $y$ covering $v$ satisfy $|\ones{y^2}-\ones{x^2}|=1$. So in a trial, $v$ can be covered by choosing $x$ as a parent (which happens with probability at least $1/|P_t|$ for both GSEMO and \NSGADYN) and flipping a one bit to zero in the second half of $x$ if $\ones{y^2} = \ones{x^2}-1$ or flipping a zero bit to one in the second half of $x$ if $\ones{y^2} = \ones{x^2}+1$, while leaving all other bits unchanged. If $\ones{y^2} \neq n/2$, a successful mutation in the former case happens with probability at least $\ones{x^2}/n \cdot (1-1/n)^{n-1} \geq \ones{x^2}/(en) = (\ones{y^2}+1)/(en)$ and if $\ones{y^2} \neq 0$, in the latter with probability at least $(n/2-\ones{x^2})/n \cdot (1-1/n)^{n-1} \geq (n/2-\ones{x^2})/(en) = (n/2-\ones{y^2}+1)/(en)$. If $\ones{y^2} \in \{0,n/2\}$, the correct bit in $x$ to cover $v$ is flipped with probability at least $1/n \cdot (1-1/n)^{n-1} \geq 1/(en) = \min\{\ones{y^2}+1,n/2-\ones{y^2}+1\}/(en)$. In either case, $v$ is covered in a single trial with probability at least $\min\{\ones{y^2}+1,n/2-\ones{y^2}+1\}/(en|P_t|)$. 

\emph{GSEMO}: Note that $1 \leq |P_t| \leq \ell$.  Since only one solution in each generation is created, and at most $\ell$ distinct fitness vectors on the Pareto front have to be covered when they are adjoining fitness vectors, GSEMO requires
\begin{align*}
&\sum_{j=\frac{n}{4}-\frac{\ell-1}{2}}^{\frac{n}{4}+\frac{\ell-1}{2}} \frac{en\ell}{\min\{j+1,n/2-j+1\}} \leq 2 \sum_{j=\frac{n}{4}}^{\frac{n}{4}+\frac{\ell-1}{2}} \frac{en\ell}{n/2-j+1} = 2en \ell \sum_{j=\frac{n}{4}-\frac{\ell-1}{2}+1}^{\frac{n}{4}+1} \frac{1}{j}\\
\end{align*}
fitness evaluations in expectation to cover the whole Pareto front. We estimate the sum on the right separately for the following two subcases (i) and (ii).

(i): Suppose that $\ell < n/4$. Then $n/4-(\ell-1)/2+1 \geq n/8$ and hence
$$\sum_{j=\frac{n}{4}-\frac{\ell-1}{2}+1}^{\frac{n}{4}+1} \frac{1}{j} \leq \sum_{j=\frac{n}{4}-\frac{\ell-1}{2}+1}^{\frac{n}{4}+1} \frac{8}{n} = \frac{\ell+1}{2} \cdot \frac{8}{n} \leq \frac{8\ell}{n}$$
and the total expected runtime to cover the Pareto front is $O(\ell^2) = O(\ell^2\ln(en/(n-2\ell+3)))$. 

(ii): Suppose that $\ell \geq n/4$. Then for $\alpha:=\max\{1,n/4-(\ell-1)/2\}$
\begin{align*}
\sum_{j=\frac{n}{4}-\frac{\ell-1}{2}+1}^{\frac{n}{4}+1} \frac{1}{j} &= \sum_{j=1}^{\frac{n}{4}+1} \frac{1}{j} - \sum_{j=1}^{\frac{n}{4}-\frac{\ell-1}{2}} \frac{1}{j} \leq \ln \Bigl(\frac{n}{4} + 1 \Bigr)+1-\ln (\alpha) = \ln \left(\frac{e(n+4)}{4\alpha}\right)\\
&= O\left(\ln \Bigl(\frac{en}{n-2\ell+3}\Bigr)\right),
\end{align*}
resulting in an expected runtime of $O(n\ell \ln(en/(n-2\ell+3))) = O(\ell^2 \ln(en/(n-2\ell+3)))$ for covering the Pareto front of $\CLIMB$. 


\emph{\NSGADYN}: Following the same arguments as for GSEMO, we obtain the following bound on the total number of fitness evaluations. The additional term $\lambda$ accounts for the remaining fitness evaluations needed to complete a generation $t$ in which an adjoining fitness vector is covered. Note also that Lemma~\ref{lem:first-ranked} applies, since $\mu_t \geq \min\{4|S_t|, \mu_t + \lambda_t\}$. Thus, once vectors on the Pareto front are covered, they cannot become uncovered. Since $\lambda_t \leq \lambda$ and $|P_t| \leq \mu_t \leq \mu$, we obtain the upper bound on the expected number of fitness evaluations
\begin{align*}
&\sum_{j=\frac{n}{4}-\frac{\ell-1}{2}}^{\frac{n}{4}+\frac{\ell-1}{2}} \left(\frac{\mu en}{\min\{j+1,n/2-j+1\}} + \lambda \right) \leq \lambda \ell + 2 \sum_{j=\frac{n}{4}}^{\frac{n}{4}+\frac{\ell-1}{2}} \frac{\mu en}{n/2-j+1} \\
&= \lambda \ell + 2 \mu en \sum_{j=\frac{n}{4}-\frac{\ell-1}{2}+1}^{\frac{n}{4}+1} \frac{1}{j}.
\end{align*}

So in a similar way as for GSEMO above, we obtain that \NSGADYN covers the Pareto front in expected $O(\lambda \ell + \mu \ell \ln(\frac{en}{n-2\ell+3}))$ fitness evaluations which completes the proof. 
\end{proof}

We are now ready to prove the main result of this section: an upper bound on the time required by GSEMO and \NSGADYN to cover the entire Pareto front of $\CLIMB$, starting from initialization. The result follows by combining all the preceding lemmas.

\begin{theorem}
\label{thm:positive-result}
Consider GSEMO or \NSGADYN optimizing $f:=\CLIMB$ for $\ell \in \{1,3, \ldots , n/2+1\}$. For \NSGADYN, let $\mu_0=1$, $\mu_{t+1}=\min\{\mu_t+\lambda_t,4|S_t|\}$, where $1 \leq \lambda_t \leq \mu_t$, and suppose that the positive threshold $r$ satisfies $r\geq 4|S|$. Then both algorithms need at most $O(n \ln(n)+\ell^2 \ln(en/(n-2\ell+3)))$ fitness evaluations in expectation to cover the Pareto front of $\CLIMB$.
\end{theorem}

\begin{proof}
We call an individual $x \in P_t$ \emph{good} if $f(x) \neq 0$, and $\ones{x^1} \geq 19n/80$. We consider two cases. 

\textbf{Case 1:} Suppose that the initialized individual is good. By Lemma~\ref{lem:initialization-success}, the probability that $x$ is initialized as a good individual equals $e^{-\Omega(n)}$. We apply the method of typical runs~\cite[Section~5.6]{Jansen2013} and split the optimization process into two phases. Some phases may be skipped if their goals have already been achieved. In the first phase, we estimate the expected number of fitness evaluations required to obtain a Pareto optimal solution. The second phase then focuses on covering the entire Pareto front.

\textbf{Phase 1: Create a Pareto optimal search point.} We show that both algorithms need $O(n \ln(n))$ fitness evaluations in expectation to finish this phase. Let $\delta_t$ be defined as in Subphase~2. Then $\delta_t \geq 19n/80$, and a Pareto optimal solution exists if $\delta_t=n/2$. Further, $\delta_t$ cannot decrease by Lemma~\ref{lem:first-ranked} for \NSGADYN, and likewise cannot decrease for GSEMO, since the latter retains only non-dominated solutions. By Lemma~\ref{lem:maximum-ones-increase}, $\delta_t$ can be increased with probability at least $(n/2-\delta_t)/(4en)$ if an individual $x \in P_t$ with $\ones{x^1} = \delta_t$ is chosen as a parent.

\emph{GSEMO:} Note that $P_t$ consists only of good individuals $x$ with $\ones{x^1} = \delta_t$ and hence such an individual is chosen with probability one as a parent. Thus, $\delta_t$ increases with probability at least $(n/2-\delta_t)/(4en)$ in each generation. Thus, setting $p_i:=(n/2-i)/(4en)$, the expected number of generations (which is one less than the number of fitness evaluations) required to complete this phase in total is at most 
$$\sum_{i=\lceil{19n/80}\rceil}^{n/2-1}\frac{1}{p_i} \leq \sum_{i=1}^{n/2-1} \frac{4en}{n/2-i} \leq \sum_{i=1}^{n/2} \frac{4en}{i} \leq 4en (\ln(n/2)+1) = O(n \ln(n))$$
where we use the harmonic-sum bound $\sum_{i=1}^k 1/i \leq \ln(k)+1$ for all $k \in \mathbb{N}$.

\emph{\NSGADYN:} Note also that $\mu_{t+1} \leq 4 |S_t|$ and hence at least $\mu_{t+1}/4$ individuals from $P_{t+1}$ are non-dominated, since by the non-dominated sorting procedure at least $\min\{\mu_{t+1},|S_t|\} = \min \{\mu_t+\lambda_t,4|S_t|,|S_t|\} = |S_t|$ non-dominated individuals are taken into $P_{t+1}$. Therefore, an individual $x$ with $\ones{x^1} = i=\delta_t$ is chosen as a parent with probability at least $1/4$. So we need $4/p_i$ fitness evaluations in expectation to create an individual $y$ with $\ones{y^1}>\delta_t$ and nonzero fitness. Note that this number is independent of the current population size $\mu_t$, since at least $\mu_t/4$ individuals from $P_t$ are always non-dominated. Denote by $t_i^*$ the generation in which $\delta_t$ increases from $i$. It remains to estimate the expected number of fitness evaluations until generation $t_i^*+1$ is reached (in which $\delta_{t_i^*+1}>i$). Note that all $x \in P_t$ have fitness distinct from zero and satisfy $\ones{x^1} \geq 19n/80$. Note also that $\mu_0=1$, since we initialize a good individual. So we can apply Lemma~\ref{lem:expected-population-bounded-NSGADYN} and obtain that the maximum population size $\mu_t^{(i)}$ over all generations $t$ for which $\delta_t=i$ satisfies $\expect{\mu_t^{(i)}}\leq 320/p_i$. Thus, after also accounting for the at most $\lambda_t \leq \mu_t$ additional offspring created in generation $t=t_i^*$, we need at most $4/p_i + 320/p_i = 324/p_i$ fitness evaluations in expectation to start in the subsequent generation $t_i^*+1$. Hence, after at most $324/p_i$ fitness evaluations in expectation, $\delta_t$ has been increased. Since $\delta_t \geq \lceil{19n/80}\rceil$,
the total expected number of fitness evaluations required to reach $\delta_t=n/2$, or equivalently the Pareto front, is at most 
$$\sum_{i=\lceil{19n/80}\rceil}^{n/2-1} \frac{324}{p_i} \leq \sum_{i=1}^{n/2-1} \frac{324}{p_i} = O(n \ln(n)).$$

\textbf{Phase 2: Cover the whole Pareto front}: We show that both algorithms need $O(\ell^2 \ln(en/(n-2\ell+3)))$ fitness evaluations in expectation to cover the Pareto front. For GSEMO, this follows from Lemma~\ref{lem:covering-front}(1). For \NSGADYN, we have that $\mu_0=1$ and $\mu_{t+1}=\min\{\mu_t+\lambda_t,4|S_t|\}$ for $t \geq 0$. Note that $|S_t| \leq \ell$ and hence $\mu_t \leq 4\ell$. Thus, we can also apply Lemma~\ref{lem:covering-front}(2) with $\mu=\lambda=4\ell$ to see that \NSGADYN needs $O(\ell^2 \ln(en/(n-2\ell+3)))$ fitness evaluations in expectation to cover the Pareto front of $\CLIMB$.

\textbf{Case 2:} Suppose that the initialized individual is not good. We then divide the optimization process into three phases, where the last two correspond to the two phases considered in the first case. In the first phase, we estimate the expected number of fitness evaluations needed to obtain a Pareto-optimal solution.

\textbf{Phase 1: All search points $x \in P_t$ have fitness distinct from zero.} 

\textbf{Phase 2: Create a Pareto optimal search point.}

\textbf{Phase 3: Cover the Pareto front.}

\begin{lemma}
\label{lem:duration-subphases}
GSEMO and \NSGADYN complete Phase~1 in expected $O(n \sqrt{n} e^{\sqrt{n} \ln(n)})$ fitness evaluations, Phase~2 in expected $O(n^2)$ fitness evaluations, and Phase~3 in expected $O(\ell^2 \ln(en/(n-2\ell+3)))$ fitness evaluations.
\end{lemma}

\begin{proof}
We consider all three subphases separately. 

\textbf{Phase 1: Create an individual with nonzero fitness.} During this phase, the maximum number of mutually incomparable non-dominated solutions is one, since all search points have fitness zero. Consequently, the population size is one for GSEMO and at most four for \NSGADYN. In the latter case, from the third generation onward, we have $\mu_t = 4$, so the population size $\mu$ is constant as long as no individual with nonzero fitness is created. Thus, Lemma~\ref{lem:initialization-drift} applies with $\mu = 1$ for GSEMO and $\mu = 4$ for \NSGADYN (from the third generation onward). After $O(n \sqrt{n})$ fitness evaluations, an individual $x$ is created with $\ones{x^2} \in W=\{x \in \{0,1\}^n \mid n/4-\sqrt{n} \leq \ones{x^2} \leq n/4+\sqrt{n}\}$. After creating $x$ in generation $T$, the probability that $x$ survives is at least $1/2$. Moreover, with probability at least $n^{-\sqrt{n}}/4 = e^{-\sqrt{n} \ln(n)}/4$, an individual $y$ with $\ones{y^2}=n/4$ is created in generation $T+1$ by choosing an individual $z \in W$ as a parent (probability at least $1/4$) and flipping at most $\lfloor{\sqrt{n}}\rfloor$ bits (probability at least $n^{-\sqrt{n}} = e^{-\sqrt{n} \ln(n)}$). Such a $y$ has nonzero fitness. If that individual $y$ is not created in $T+1$ generations, we repeat. The expected number of repetitions required is at most $O(e^{\sqrt{n} \ln(n)})$, and each repetition requires $O(n \sqrt{n})$ fitness evaluations in expectation. By Wald's equation, the expected number of fitness evaluations required to finish this subphase is at most $O(n \sqrt{n} e^{\sqrt{n} \ln(n)})$.

\textbf{Phase 2: Create a Pareto optimal individual.} We consider both algorithms separately.\\
\emph{GSEMO:} As in the Case 1, where only one good individual is initialized, $P_t$ consists only of good individuals $x$ with $\ones{x^1} = \delta_t$ and hence, such an individual is chosen with probability one as a parent. Thus, $i=\delta_t$ increases with probability at least $p_i=(n/2-i)/(4en)$ in each generation. So the expected number of generations to complete this phase is at most $\sum_{i=1}^{n/2-1} \frac{1}{p_i} = O(n \ln(n))$.

\emph{\NSGADYN:} As in Case~1, we have $\mu_{t+1} \leq 4 |S_t|$ and hence at least $\mu_{t+1}/4$ individuals from $P_{t+1}$ are non-dominated. Therefore, an individual $x$ with $\ones{x^1} = i=\delta_t$ is chosen as a parent with probability at least $1/4$. So we need $4/p_i$ fitness evaluations in expectation to create an individual $y$ with $\ones{y^1}>\delta_t$ and nonzero fitness. To finish the iteration where such a $y$ is created, we need at most $\lambda_t \leq \mu_t \leq 2n+4$ additional fitness evaluations. So the total expected number of fitness evaluations required to complete this phase is at most 
$$\sum_{i=1}^{n/2-1} \Bigl(\frac{4}{p_i} + 2n + 4\Bigr) = O(n \ln(n) + n^2) = O(n^2).$$

\textbf{Phase 3: Cover the Pareto front.} Note also, for \NSGADYN, we have that $\mu_t \leq 4\ell$ for all $t \geq 0$. Thus, this phase can be analyzed in the same way as Phase~2 in Case~1. Consequently, both algorithms require an expected number of $O(\ell^2 \ln(en/(n-2\ell+3)))$ fitness evaluations to cover the Pareto front of $\CLIMB$.
\end{proof}

By summing the runtimes of both phases from Case~1 and the three phases from Case~2, we conclude with Lemma~\ref{lem:duration-subphases} that the total expected number of fitness evaluations required for both algorithms to cover the Pareto front is at most 
\begin{align*}
O\Bigl(e^{-\Omega(n)}&\Bigl(n \sqrt{n} e^{\sqrt{n} \ln(n)} + n^2\Bigr) + (1-e^{-\Omega(n)}) n \ln(n) + \ell^2 \ln \Bigl(\frac{en}{n-2\ell+3}\Bigr)\Bigr)\\
&= O\Bigl(n \ln(n) + \ell^2 \ln \Bigl(\frac{en}{n-2\ell+3}\Bigr)\Bigr)
\end{align*}
which proves the theorem.
\end{proof}

To complement the established upper runtime bounds for GSEMO and \NSGADYN, we also provide lower runtime bounds for reaching the Pareto front of $\CLIMB$. Note that this process relies solely on hill climbing. Consequently, for $\ell = O(\sqrt{n \ln n})$, the runtime bounds for GSEMO and \NSGADYN on $\CLIMB$ are tight.

\begin{theorem}
\label{thm:Dynamic-lower-bound}
Both GSEMO and \NSGADYN, with $\mu_0=1$, any choice of $g_t$ and $h_t$, and any threshold $r \geq 1$, require $\Omega(n\ln(n))$ fitness evaluations in expectation to reach the Pareto front when optimizing $f:=\CLIMB$ for $\ell \in \{1,3,\ldots,n/2+1\}$.
\end{theorem}

\begin{proof}
We first show that each of the following three failure events occurs with probability at most $o(1)$. Conditioned on none of these events occurring, we then show that creating a Pareto-optimal individual requires $\Omega(n \ln n)$ fitness evaluations in expectation. Thus, the overall expected number of fitness evaluations required by either algorithm to optimize $\CLIMB$ is at least $(1+o(1))\Omega(n \ln n) = \Omega(n \ln n)$.

The first failure event is that the initial search point $x$ satisfies $\ones{x^1} \geq n - 3\sqrt{n}$. For both algorithms, this occurs with probability at most $e^{-\Omega(n)}$. The second failure event is that, within $O(n \ln n)$ fitness evaluations, at least $2\sqrt{n}$ bits in the first half of some search point $x \in P_t$ are flipped. By a union bound, this happens with probability at most $$O(n \ln(n)) \binom{n/2}{\lceil{2 \sqrt{n}}\rceil}/(n^{\lceil{2 \sqrt{n}}\rceil}) = O(n \ln(n)/n^{\sqrt{n}}) = o(1).$$ 
Compare also with the proof of Lemma~\ref{lem:initialization-drift} for a similar argument. A third failure event is that, for $m \in \mathbb{N}$ with $m \leq \lceil{n-3\sqrt{n}}\rceil$, at least $3+m$ zero bits are flipped in the first half of a search point $x$ when $n-3\sqrt{n}-m \leq \ones{x^1}$ within $O(n \ln (n))$ fitness evaluations, happening with probability at most 
$$O(n \ln(n)) \cdot \binom{\lceil{3\sqrt{n}}\rceil+m}{3+m} /n^{3+m} \leq O(n \ln(n)) \cdot \binom{\lceil{3\sqrt{n}}\rceil}{3} /n^3 = O(n \ln(n) \cdot n^{-1.5}) = o(1).$$ 
So, suppose that none of these occur. Let 
$$\delta_t:=\max\{i \in \{0, \ldots , n/2\} \mid \text{there is $x \in P_t$ with $\ones{x^1}=i$}\}.$$
Note that there is a time $t_0$ where $n-3\sqrt{n} \leq \delta_{t_0} \leq n-\sqrt{n}$. Then $\delta_t$ for $t \geq t_0$ can only be increased by at most two, and that with probability at most $\delta_t/n$: If an individual $x$ with $\ones{x^1}=\delta_t$ is selected as a parent, then one zero bit is flipped in the first half of $x$ with probability at most $\delta_t/n$. But also if $\ones{x^1}<\delta_t$ the probability to flip at least $\delta_t-\ones{x^1}+1$ zero bits in the first half of $x$ is not larger than $\delta_t/n$. So, in expectation, at least $n/\delta_t$ fitness evaluations are necessary to increase $\delta_t$, and if $\delta_t$ is increased, then by at most two. Also, it happens that $\delta_t \in \{n/2-2i,n/2-2i+1\}$ for each $i \in [\lfloor{\sqrt{n}/2}\rfloor]$, and each such $\delta_t$ has to be increased to further progress towards the Pareto front. So the expected number of fitness evaluations to increase $\delta_t$ to $n$ is at least $\sum_{i=1}^{\lfloor{\sqrt{n}/2}\rfloor} n/(2i) \geq n/2 \cdot \ln(\lfloor{\sqrt{n}/2}\rfloor) = \Omega(n \ln(n))$.
\end{proof}

\section{Runtime Analysis of the Vanilla NSGA-II on $\CLIMB$} 

In this section, we first provide an upper bound on the runtime of \NSGADYN with fixed population size $\mu \geq \ell$ and offspring population size $\lambda$ for the optimization of $\CLIMB$. The proof of this result combines arguments from~\citep{TightOneMax} for analyzing the hill climbing process with those from~\citep{ZhengLuiDoerrAAAI22} for spreading solutions across the Pareto front and covering it.

\begin{theorem}
\label{thm:Vanilla-upper-bound}
Consider $f:=\CLIMB$ for $\ell \in \{1,3, \ldots , n/2\}$. Let $\mu,\lambda \in \mathbb{N}$ with $\mu \geq 4\ell$. Then, \NSGADYN with $\mu_t=\mu$, $\lambda_t=\lambda$, and any threshold $r \geq \mu$, requires $O((\mu + \lambda)n+ n\ln(n) + \mu \ell\ln(en/(n-2\ell+3)))$ fitness evaluations in expectation to cover the Pareto front of $f$. For $\mu = 4\ell$ and $1 \leq \lambda \leq \mu$, the stated runtime bound becomes $O(\ell n+ n\ln(n) + \ell^2\ln(en/(n-2\ell+3)))$.
\end{theorem}

\begin{proof}
We derive the optimization process in two phases. The first phase is about reaching the Pareto front, and the second is about covering it. For reaching the Pareto front, define $\delta_t:=\max\{i \in \{0, \ldots , n/2-1\} \mid \text{there is $x \in P_t$ with $\ones{x^1}=i$ and $f(x) \neq 0$}\}$. To increase $\delta_t$, we first create at least $\min\{\mu,\lceil{n/(n/2-\delta_t)}\rceil\}$ solutions $x$ with $\ones{x^1}=\delta_t$ and fitness distinct from zero and then, increase $\delta_t$ by applying a mutation. 
For the first case, suppose there are currently $i$ such individuals $x$. In one trial, the probability of increasing their number by one is at least $i/\mu \cdot 1/4 = i/(4\mu)$, since we choose one of these individuals with probability at least $1/\mu$, and then clone it with probability at least $(1-1/n)^n \geq 1/4$. Hence, the expected number of trials to double the number of individuals $x$, or to reach $\mu$ such individuals, is at most $8\mu$. Consequently, the expected number of generations is $\lceil 8\mu/\lambda\rceil \leq 8\mu/\lambda + 1$. Let $\gamma_t:=n/2-\delta_t$. Then the expected number of generations until there are $\min\{\mu,\lceil{n/\gamma_t}\rceil\}$ such solutions $x$ is $(8 \mu/\lambda + 1)(\lceil{\log(\min\{\mu,\lceil{n/\gamma_t}\rceil\})}\rceil) \leq (8 \mu/\lambda + 1)(\log(\min\{\mu,\lceil{n/\gamma_t}\rceil\})+1) \leq (8 \mu/\lambda + 1)(\log(n/\gamma_t)+2)$  where log denotes the logarithm to base $2$. Then, increasing $\delta_t$ happens with probability at least $\min\{\mu,\lceil{n/\gamma_t}\rceil\}/\mu \cdot \gamma_t/n \geq \min\{\gamma_t/n,1/\mu\}$. Thus, the expected number of trials required to increase $\delta_t$ after creating $\lceil{n/\gamma_t}\rceil$ individuals is at most $\mu + n/\gamma_t$. This results in at most
$\lceil{\mu/\lambda + n/(\lambda\gamma_t)}\rceil \leq \mu/\lambda + n/(\lambda\gamma_t) + 1$
generations. Therefore, the expected number of generations until the Pareto front is reached is bounded from above by
\begin{align*}
& \sum_{i=1}^{n/2} \Bigl(\Bigl(\frac{8 \mu}{\lambda} + 1\Bigr)\Bigl(\log \Bigl(\frac{n}{i}\Bigr)+2\Bigr) + \frac{\mu}{\lambda}+\frac{n}{\lambda i} + 1 \Bigr) \\
&= \Bigl(\frac{8 \mu}{\lambda} + 1\Bigr)\Bigl(n+ \sum_{i=1}^{n/2}\log \Bigl(\frac{n}{i}\Bigr)\Bigr) + \frac{n\mu}{2\lambda} + \frac{n}{2} + \sum_{i=1}^{n/2} \frac{n}{\lambda i}\\
&\leq \Bigl(\frac{8 \mu}{\lambda} + 1\Bigr)\Bigl(n+ \log \Bigl(\frac{n^{n/2}}{(n/2)!}\Bigr)\Bigr) + \frac{n\mu}{2\lambda} + \frac{n}{2} + \frac{n (\ln(n)+1)}{\lambda}\\
&\leq \Bigl(\frac{8 \mu}{\lambda} + 1\Bigr)\Bigl(n+ \log \Bigl(\frac{(en)^{n/2}}{(n/2)^{n/2}}\Bigr)\Bigr) + \frac{n\mu}{2\lambda} + \frac{n}{2} + \frac{n (\ln(n)+1)}{\lambda}\\
&\leq \Bigl(\frac{8 \mu}{\lambda} + 1\Bigr)\Bigl(n+ \frac{n}{2}\log(2e)\Bigr) + \frac{n\mu}{2\lambda} + \frac{n}{2} + \frac{n (\ln(n)+1)}{\lambda}\\
&= O\left(n+\frac{n\mu}{\lambda}+\frac{n\ln(n)}{\lambda}\right).
\end{align*}
So the expected number of fitness evaluations until the Pareto front is reached in $O(n\lambda+n\mu+ n\ln(n))$. 

So it remains to cover the Pareto front. To this end, we can apply Lemma~\ref{lem:covering-front}(2), since $\mu \geq 4\ell \geq \min\{\mu+\lambda,4|S_t|\}$, and obtain that the Pareto front of $\CLIMB$ is covered in $O(\lambda \ell + \mu\ell\ln(en/(n-2\ell+3)))$ fitness evaluations. This proves the theorem by summing the runtimes for the hill climbing phase and the phase covering the Pareto front. 
\end{proof}

However, the main result of this section is a lower bound on the number of fitness evaluations this version of \NSGADYN with fixed population size $\mu$ and offspring population size $\lambda$ requires to optimize $\CLIMB$, in particular, to find a Pareto-optimal search point. This fixed population size may hinder efficient optimization: $\mu$ must be at least as large as the size of the Pareto front, forcing the algorithm to maintain a population of size at least $\ell$ from the start. Consequently, many fitness evaluations are spent on individuals with low fitness, which slows down progress towards the Pareto front. To formalize this, we adapt the family-tree argument from~\citep{WittLower}, similar to its use as in~\citep{SUDHOLT20092511,TightOneMax}. We define a family forest consisting of $\mu$ family trees. Each node stores a triple: the individual, the generation $t$ in which it was created, and an index $i \in \{0,1,\ldots,\lambda\}$ indicating that the individual was created during the $i$-th iteration of the offspring loop (Line~5 of Algorithm~\ref{alg:nsga-ii}), where $i=0$ indicates an initialized individual. Each node is linked to its parent. The forest is defined inductively. In generation $0$, it consists of $\mu$ single-node trees, one for each individual in $P_0$, with all components of the triple equal to $0$. Suppose the forest for $P_t$ is given. During generation $t$, whenever a parent $x$ creates an offspring $y$ during iteration $i$ of the offspring loop, we add a new node $(y, t+1, i)$ to the same tree as $x$ and connect it to $x$. After all offspring have been created, this gives the forest for $P_{t+1}$. The survival selection in generation $t+1$ does not alter the forest.

\begin{theorem}
\label{thm:negative-result}
For $1 \leq \mu =  O(\text{poly}(n))$ and $1 \leq \lambda = O(\text{poly}(n))$ the \NSGADYN with fixed population size $\mu$ and offspring population size $\lambda$ (if $g=\mu$ and $h=\lambda$) needs at least $\Omega(\mu n)$ fitness evaluations in expectation to find a Pareto optimal search point of $f:=\CLIMB$ for every $\ell \in \{1,3,5, \ldots , n/2+1\}$. For $\mu=\ell$, the expected number of fitness evaluations becomes $\Omega(\ell n)$.
\end{theorem}

\begin{proof}
By Lemma~\ref{lem:initialization-success} every individual $x$ is initialized with $(\ones{x^1},\ones{x^2}) \in [\frac{19n}{80},\frac{21n}{80}] \times [\frac{n}{8},\frac{3n}{8}]$ with probability at least $1-\mu e^{-\Omega(n)} = 1-e^{-\Omega(n)}$. These individuals have nonzero fitness. Suppose this happens. Let $P_0:=\{x_0^1, \ldots , x_0^{\mu}\}$. Denote by $\text{tree}_t(x_0^i)$ the family tree with root $x_0^i$. First, we show that after $t$ generations, we have $\Pr(\text{depth}(\text{tree}_t(x_0^i)) \geq 3t\lambda/\mu) \leq \mu e^{-\Omega(t \lambda/\mu)}$ for all $i \in [\mu]$. To this end, we consider the probability of generating a path of length $k$ consisting of nodes with labels $(t_1,i_1), \ldots , (t_k,i_k)$ as the second and third components for $0 \leq t_1 < \ldots < t_k \leq t$ and $i_1, \ldots , i_k \in [\lambda]$. Note that there are $\binom{t}{k}$ ways to label such a path by the second component, and $\lambda^k$ ways to label it with respect to the third component. So the number of such paths of length $k$ is at most $\binom{t}{k} \cdot \lambda^k$. For $m>1$, the probability of extending such a path by a node with label $(t_m,i_m)$ is at most $1/\mu$, since the corresponding individual of the node with label $(t_{m-1},i_{m-1})$ must be chosen as a parent in the $i_m$-th iteration of the for-loop in generation $t_m$. Thus, the probability of creating such a path is at most $(1/\mu)^{k-1}$.  Hence, by a union bound, the probability of creating any path of length $k$ in $t$ generations is at most $\mu \cdot \binom{t}{k} \cdot (\lambda/\mu)^k$. If $k \geq 3t\lambda/\mu$, then, using the standard bound $\binom{t}{k} \leq (et/k)^k$,
\begin{align*}
\mu \binom{t}{k} \frac{\lambda^k}{\mu^k} \leq \mu \left(\frac{et \lambda}{k\mu}\right)^k \leq \mu \left(\frac{e}{3}\right)^{k} = \mu e^{-\Omega(t \lambda/\mu)}.
\end{align*}
Suppose that this does not happen. In particular, all $\mu$ trees in the forest have depth at most $\alpha:= \min\{t,3t\lambda/\mu\}$. So $k \leq \alpha$. We then bound the Hamming distance between $y_0:=x_0^i$ and solutions corresponding to nodes in the tree with root $y_0$ as follows. We call a path to a leaf at time $t$ \emph{bad} if a solution with Hamming distance of at least $n/5$ to $y_0$ in the second half of $y_0$ is created along that path. When adding the $m$-th node with triple $(y_m,t_m,i_m)$ for some $i_m \in [\lambda]$ to a path, we see that $\expect{H(y_{m-1}^2,y_m^2)} = 1/2$ since the expected number of flipped bits in the second half of $y_{m-1}$ is $1/2$, when applying standard bit mutation. Hence, the expected change in Hamming distance along that path is at most $k/2 \leq 3t\lambda/(2\mu)$. For $t \leq 4n \mu/(54\lambda)$, we obtain that the expected change in Hamming distance is at most $n/9$. Thus, using a standard Chernoff bound, for $\delta = 4/5$, every $m \in [k]$, and $t \leq 4n\mu/(54\lambda)$, we have that $\Pr(H(y_0^2,y_m^2) \geq n/5) = \Pr(H(y_0^2,y_m^2) \geq (1+\delta) n/9) \leq e^{-\expect{X} \delta^2/3} \leq e^{-16n/675}$. By setting $t=\lfloor{4n\mu/(1350 \lambda)}\rfloor$, we obtain by a union bound over all possible paths of length at most $\alpha \leq 12n/1350$ and all nodes on such a path that the probability of generating a bad path is at most
\[
\alpha \sum_{k=1}^{\min\{\alpha,t\}} \mu \binom{t}{k} \frac{\lambda^k}{\mu^k} \cdot e^{-16n/675} \leq \alpha \mu \sum_{k=1}^{\alpha} \left(\frac{4en}{1350k}\right)^k \cdot e^{-16n/675} = \alpha \mu e^{-\Omega(n)} = e^{-\Omega(n)}
\]
within $t$ generations. This concludes the proof, since the existence of such a path is necessary for creating a Pareto optimal search point.
\end{proof}

\section{Conclusions and Discussions}
To explore the benefits of variable population sizes, we developed \NSGADYN, a variant of \nsga that dynamically adjusts the population size. Our results show that this flexibility can improve performance, especially during the hill-climbing phase. Specifically, we showed that both GSEMO and \NSGADYN achieve superconstant speedups over the vanilla \nsga on $\CLIMB$ across a broad range of values of $\ell$. This contrasts sharply with classical pseudo-Boolean benchmarks such as $\LOTZ$, $\OMM$, and $\COCZ$, for which only matching upper bounds on the expected running times of GSEMO and \nsga are known. To the best of our knowledge, this is the first result showing a superconstant speedup of GSEMO over vanilla \nsga on a bi-objective pseudo-Boolean problem. Similar advantages may hold for other MOEAs with fixed population sizes, such as SPEA-2, SMS-EMOA, or \nsgaIII. However, our work has several limitations that point to interesting directions for future research. First, our analysis is limited to $\CLIMB$. It remains unclear how MOEAs with dynamic populations behave on more complex benchmark problems and whether dynamic population sizes can affect not only hill climbing but also the exploration of different regions of the search space, potentially leading to exponential speedups. Second, we only considered specific update rules for adjusting the population size in \NSGADYN. One could also investigate alternative choices of $g$ and $h$ and how they may accelerate the optimization process.
Another direction for future work is to consider different MOEAs, such as PAES25, and dynamize the size of their archives~\citep{Opris2025PAES}. Overall, alongside~\citep{ijcai2026p698}, this work represents an initial step towards understanding the role of dynamic population sizes in evolutionary multi-objective optimization. We hope that the insights gained will be valuable for both theoretical analysis and practical applications. In particular, they may guide the development of improved variants of our proposed \NSGADYN, further enhancing the performance of MOEAs on complex real-world problems.

\bibliographystyle{plainnat}
\bibliography{neurips26}

@article{WANG2022101104,
title = {Evolutionary Algorithm with Dynamic Population Size for Constrained Multiobjective Optimization},
journal = {Swarm and Evolutionary Computation},
volume = {73},
pages = {101104},
year = {2022},
noissn = {2210-6502},
nodoi = {https://doi.org/10.1016/j.swevo.2022.101104},
nourl = {https://www.sciencedirect.com/science/article/pii/S2210650222000748},
author = {Bing-Chuan Wang and Zhong-Yi Shui and Yun Feng and Zhongwei Ma}
}

@article{Liang2023,
    author = {Liang, Jing and Chen, Zhaolin and Wang, Yaonan and Ban, Xuanxuan and Qiao, Kangjia and Yu, Kunjie},
    title = {A dual-population constrained multi-objective evolutionary algorithm with variable auxiliary population size},
    journal = {Complex \& Intelligent Systems},
    volume = {9},
    number = {5},
    pages = {5907-5922},
    year = {2023},
    nodoi = {10.1007/s40747-023-01042-2},
    nourl = {https://doi.org/10.1007/s40747-023-01042-2}
}

@article{TIAN2025129296,
title = {Adaptive population sizing for multi-population based constrained multi-objective optimization},
journal = {Neurocomputing},
volume = {621},
pages = {129296},
year = {2025},
nodoi = {https://doi.org/10.1016/j.neucom.2024.129296},
nourl = {https://www.sciencedirect.com/science/article/pii/S0925231224020678},
author = {Ye Tian and Ruiqin Wang and Yajie Zhang and Xingyi Zhang}
}

@article{NSGAMACHINE,
title = {Hybrid machine learning and {NSGA-II} optimization of rail anti-climbing systems: a multi-criteria design framework},
journal = {Structural and Multidisciplinary Optimization},
volume = {69},
number = {17},
year = {2025},
noissn = {1359-8368},
nodoi = {https://doi.org/10.1016/j.compositesb.2025.112961},
nourl = {https://www.sciencedirect.com/science/article/pii/S1359836825008674},
author = {Izanloo Mehri and Shahravi Majid and Khalkhali Abolfazl}
}

@article{WittLower,
    author = {Witt, Carsten},
    title = {Runtime Analysis of the $(\mu + 1)$ {EA} on Simple {P}seudo-{B}oolean Functions},
    journal = {Evolutionary Computation},
    volume = {14},
    number = {1},
    pages = {65-86},
    year = {2006},
    nodoi = {10.1162/evco.2006.14.1.65},
    nourl = {https://doi.org/10.1162/evco.2006.14.1.65}
}

@article{NSGANEURAL,
title = {Physics-Informed Neural Network with {NSGA II} and Levenberg–Marquardt Method for Kinetic Modeling in Heavy Oil Hydrocracking},
journal = {Industrial and Engineering Chemistry Research},
volume = {64},
pages = {19624 -- 19640},
number = {40},
year = {2025},
noissn = {1359-8368},
nodoi = {https://doi.org/10.1016/j.compositesb.2025.112961},
nourl = {https://www.sciencedirect.com/science/article/pii/S1359836825008674},
author = {Shahla Alizadeh and Souvik Ta and Lakshminarayanan Samavedham},
publisher = {American Chemical Society}
}

@inproceedings{ZhengLuiDoerrAAAI22,
  author    = {Weijie Zheng and
               Yufei Liu and
               Benjamin Doerr},
  title     = {A First Mathematical Runtime Analysis of the Non-dominated Sorting
               Genetic Algorithm {II} {(NSGA-II)}},
  booktitle = {Proceedings of the {AAAI} Conference on Artificial Intelligence, {AAAI}~2022},
  pages     = {10408-10416},
  publisher = {{AAAI} Press},
  year      = {2022}
}

@inproceedings{QianYZ15nips,
  author       = {Chao Qian and
                  Yang Yu and
                  Zhi{-}Hua Zhou},
  OPTeditor       = {Corinna Cortes and
                  Neil D. Lawrence and
                  Daniel D. Lee and
                  Masashi Sugiyama and
                  Roman Garnett},
  title        = {Subset selection by Pareto optimization},
  booktitle    = {Neural Information Processing Systems, NIPS 2015},
  OPTbooktitle    = {Advances in Neural Information Processing Systems 28: Annual Conference
                  on Neural Information Processing Systems 2015, December 7-12, 2015,
                  Montreal, Quebec, Canada},
  pages        = {1774--1782},
  year         = {2015},
}

@article{ClearingI,
    author = {Osuna, Edgar Covantes and Sudholt, Dirk},
    title = {On the Runtime Analysis of the Clearing Diversity-Preserving Mechanism},
    journal = {Evolutionary Computation},
    volume = {27},
    number = {3},
    pages = {403-433},
    year = {2019}
}

@article{Diversity,
    author = {Lengler, Johannes and Opris, Andre and Sudholt, Dirk},
    title = {Analysing Equilibrium States for Population Diversity},
    journal = {Algorithmica},
    volume = {27},
    number = {3},
    pages = {2317-2351},
    volume = {86},
    issue = {7},
    year = {2024},
}

@inproceedings{ijcai2026p698,
  title     = {Speeding Up the {NSGA-II} via Dynamic Population Sizes},
  author    = {Doerr, Benjamin and Krejca, Martin S. and Wietheger, Simon},
  booktitle = {Proceedings of the Thirty-Fifth International Joint Conference on
               Artificial Intelligence},
  series    = {IJCAI 2026},
  publisher = {ijcai.org},
  pages     = {6273--6281},
  year      = {2026},
}

@inproceedings{DinotDHW23,
  author    = {Matthieu Dinot and Benjamin Doerr and Ulysse Hennebelle and Sebastian Will},
  title     = {Runtime analyses of multi-objective evolutionary algorithms in the presence of noise},
  booktitle = {International Joint Conference on Artificial Intelligence, {IJCAI} 2023},
  pages        = {5549--5557},
  publisher    = {ijcai.org},
  year         = {2023},
  OPTdoi          = {10.24963/ijcai.2023/616},
}

@inproceedings{QianZTY18,
  title={On multiset selection with size constraints},
  author={Qian, Chao and Zhang, Yibo and Tang, Ke and Yao, Xin},
  booktitle={Conference on Artificial Intelligence},
  series = {AAAI 2018},
  pages={1395--1402},
  publisher = {AAAI Press},
  year={2018}
}

@inproceedings{DaOp2023,
  author       = {Duc-Cuong Dang and
                  Andre Opris and
                  Bahare Salehi and
                  Dirk Sudholt},
  noeditor       = {Sara Silva and
                  Lu{\'{\i}}s Paquete},
  title        = {Analysing the Robustness of {NSGA-II} under Noise},
  booktitle    = {Proceedings of the Genetic and Evolutionary Computation Conference ({GECCO}'23)},
  pages        = {642--651},
  publisher    = {{ACM} Press},
  year         = {2023},
  nourl          = {https://doi.org/10.1145/3583131.3590421},
  nodoi          = {10.1145/3583131.3590421}
}

@article{MARTINEZCOMESANA2023106770,
title = {Optimisation of {LSTM} neural networks with {NSGA-II} and {FDA} for {PV} installations characterisation},
journal = {Engineering Applications of Artificial Intelligence},
volume = {126},
pages = {106770},
year = {2023},
noissn = {0952-1976},
nodoi = {https://doi.org/10.1016/j.engappai.2023.106770},
nourl = {https://www.sciencedirect.com/science/article/pii/S0952197623009545},
author = {Miguel Martínez-Comesaña and Javier Martínez-Torres and Pablo Eguía-Oller}
}

@article{CHOWDHURY2026109365,
title = {Machine learning-assisted {NSGA-II}-{TOPSIS} optimization of weld strength and processing time in fused deposition modeling},
journal = {Results in Engineering},
volume = {29},
pages = {109365},
year = {2026},
author = {Fahmida Khatun Chowdhury and Md Sadman Sami and Md. Jahedul Alam and Ahmed Sayem and Mohammad Muhshin Aziz Khan},
}

@article{ZHANG2025112961,
title = {Machine learning-based multi-objective optimization of ceramic composite armor plates via the {NSGA-II} genetic algorithm},
journal = {Composites Part B: Engineering},
volume = {307},
pages = {112961},
year = {2025},
noissn = {1359-8368},
nodoi = {https://doi.org/10.1016/j.compositesb.2025.112961},
nourl = {https://www.sciencedirect.com/science/article/pii/S1359836825008674},
author = {Xinzhe Zhang and Rentao Wang and Chuankun Zang and Kai Song and Xiaolu Wang and Guoju Li}
}

@article{TightOneMax,
author = {Denis Antipov and Benjamin Doerr},
title = {A Tight Runtime Analysis for the ${(\mu + \lambda )}$ {EA}},
journal = {Algorithmica},
year = {2021},
nodoi = {https://doi.org/10.1002/bit.70119},
nourl = {https://analyticalsciencejournals.onlinelibrary.wiley.com/doi/abs/10.1002/bit.70119},
pages = {1054-1095},
volume = {83},
issue = {4}
}

@inproceedings{Thierens03,
  author    = {Dirk Thierens},
  OPTeditor    = {Carlos M. Fonseca and Peter J. Fleming and
               Eckart Zitzler and Kalyanmoy Deb and Lothar Thiele},
  title     = {Convergence time analysis for the multi-objective counting ones problem},
  booktitle = {Evolutionary Multi-Criterion Optimization},
  series = {EMO 2003},
  OPTbooktitle = {Evolutionary Multi-Criterion Optimization, Second International Conference, {EMO} 2003, Faro, Portugal, April 8-11, 2003, Proceedings},
  pages     = {355--364},
  publisher = {Springer},
  year      = {2003},
}

@inproceedings{LaumannsTZWD02,
  author    = {Marco Laumanns and
               Lothar Thiele and
               Eckart Zitzler and
               Emo Welzl and
               Kalyanmoy Deb},
  title     = {Running time analysis of multi-objective evolutionary algorithms on
               a simple discrete optimization problem},
  booktitle = {Parallel Problem Solving from Nature,  {PPSN} 2002},
  pages     = {44--53},
  publisher = {Springer},
  year      = {2002},
  OPTdoi       = {10.1007/3-540-45712-7\_5},
}

@article{MultiVehicles,
    author = {Liao Xingtao and Li Qing and Yang Xujing and Zhang Weigang},
    title = {Multiobjective Optimization for Crash Safety Design of Vehicles Using Stepwise Regression Model},
    journal = {Structural and Multidisciplinary Optimization},
    year = {2008},
    pages = {561–569},
    volume = {35},
    Issue = {6},
}

@InProceedings{10.1007/978-3-540-70928-2_55,
author={K{\"o}ppen, Mario
and Yoshida, Kaori},
title={Substitute Distance Assignments in {NSGA-II} for Handling Many-Objective Optimization Problems},
booktitle={Evolutionary Multi-Criterion Optimization},
year={2007},
publisher={Springer},
noaddress={Berlin, Heidelberg},
pages={727--741}
}

@ARTICLE{704576,
  author={Ishibuchi, H. and Murata, T.},
  journal={IEEE Transactions on Systems, Man, and Cybernetics, Part C (Applications and Reviews)}, 
  title={A Multi-Objective Genetic Local Search Algorithm and its Application to Flowshop Scheduling}, 
  year={1998},
  volume={28},
  number={3},
  pages={392-403},
  nodoi={10.1109/5326.704576}}

@article{LUUKKONEN2023102537,
title = {Artificial Intelligence in Multi-Objective Drug Design},
journal = {Current Opinion in Structural Biology},
volume = {79},
pages = {102537},
year = {2023},
nodoi = {https://doi.org/10.1016/j.sbi.2023.102537},
nourl = {https://www.sciencedirect.com/science/article/pii/S0959440X23000118},
author = {Sohvi Luukkonen and Helle W. {van den Maagdenberg} and Michael T.M. Emmerich and Gerard J.P. {van Westen}}
}

@article{OPRIS2026,
title = {Many-objective problems where crossover is provably essential},
journal = {Artificial Intelligence},
volume = {350},
pages = {104453},
year = {2026},
noissn = {0004-3702},
nodoi = {https://doi.org/10.1016/j.artint.2025.104453},
nourl = {https://www.sciencedirect.com/science/article/pii/S0004370225001729},
author = {Andre Opris}
}

@inproceedings{li2026why,
title={Why Popular {MOEA}s Are Popular: Proven Advantages in Approximating the Pareto Front},
author={Mingfeng Li and Qiang Zhang and Weijie Zheng and Benjamin Doerr},
booktitle={The Thirty-Ninth Annual Conference on Neural Information Processing Systems},
year={2025},
url={https://openreview.net/forum?id=8OvST1bejm}
}

@article{Jun2004,
  author    = {Jun He and Xin Yao},
  title     = {A study of drift analysis for estimating
              computation time of evolutionary algorithms},
  journal   = {Natural Computing},
  year      = {2004},
  volume    ={3},
  pages     ={21-35},
}

@inproceedings{UpBian,
author = {Bian, Chao and Zhou, Yawen and Li, Miqing and Qian, Chao},
title = {Stochastic population update can provably be helpful in multi-objective evolutionary algorithms},
year = {2023},
booktitle = {Proceedings of the Thirty-Second International Joint Conference on Artificial Intelligence},
series = {IJCAI 2023},
publisher = {ijcai.org},
pages = {5513--5521}
}

@inproceedings{ApproximationSPEA2,
author = {Alghouass, Yasser and Doerr, Benjamin and Krejca, Martin S. and Lagmah, Mohammed},
title = {Proven approximation guarantees in multi-objective optimization: {SPEA2} beats {NSGA-II}},
year = {2025},
booktitle = {Proceedings of the Thirty-Fourth International Joint Conference on Artificial Intelligence},
series = {IJCAI 2025},
publisher = {ijcai.org},
pages = {8833--8841}
}

@inproceedings{ApproximationNSGAIII,
author = {Deng, Renzhong and Zheng, Weijie and Doerr, Benjamin},
title = {The first theoretical approximation guarantees for the non-dominated sorting genetic algorithm {III} ({NSGA\nobreakdash-III})},
year = {2025},
noisbn = {978-1-956792-06-5},
nourl = {https://doi.org/10.24963/ijcai.2025/986},
nodoi = {10.24963/ijcai.2025/986},
booktitle = {Proceedings of the Thirty-Fourth International Joint Conference on Artificial Intelligence},
articleno = {986},
pages = {8867--8875},
numpages = {9},
location = {Montreal, Canada},
series = {IJCAI 2025},
publisher = {ijcai.org}
}

@inproceedings{ijcai2026p705,
  title     = {On the Impact of Crossover in Many-Objective Optimization: A Runtime Analysis of {NSGA-III}},
  author    = {Opris, Andre},
  booktitle = {Proceedings of the Thirty-Fifth International Joint Conference on
               Artificial Intelligence},
  series = {IJCAI 2026},
  publisher = {ijcai.org},
  pages     = {6334--6342},
  year      = {2026}
}

@inproceedings{OprisMultimodal,
    author       = {Andre Opris},
    title        = {A First Runtime Analysis of {NSGA-III} on a Many-Objective Multimodal Problem:                    Provable Exponential Speedup via Stochastic Population Update},
    booktitle    = {Proceedings of the Thirty-Fourth International Joint Conference on
                    Artificial Intelligence},
    pages        = {8903--8911},
    publisher    = {ijcai.org},
    year         = {2025},
    series       = {IJCAI 2025},
    nourl          = {https://doi.org/10.24963/ijcai.2023/612},
    nodoi          = {10.24963/IJCAI.2023/612},
}

@inproceedings{Giel2003,
  author    = {Oliver Giel},
  title     = {Expected runtimes of a simple multi-objective evolutionary algorithm},
  booktitle = {Proceedings of the {IEEE} Congress on Evolutionary Computation},
  series    = {CEC 2003},
  pages     = {1918-1925},
  publisher = {{IEEE} Press},
  year      = {2003}
}

@article{Laumanns2004,
  author    = {Marco Laumanns and Lothar Thiele and Eckart Zitzler},
  title     = {Running Time Analysis of Multiobjective Evolutionary Algorithms
               on {P}seudo-{B}oolean Functions},
  journal   = {{IEEE} Transactions on Evolutionary Computation},
  volume    = {8},
  number    = {2},
  pages     = {170-182},
  year      = {2004}
}

@inproceedings{Zheng2022,
  author    = {Weijie Zheng and
               Benjamin Doerr},
  noeditor    = {Jonathan E. Fieldsend and
               Markus Wagner},
  title     = {Better Approximation Guarantees for the {NSGA-II} by Using the Current
               Crowding Distance},
  booktitle = {Proceedings of the Genetic and Evolutionary Computation Conference},
  series = {GECCO 2022},
  pages     = {611--619},
  publisher = {{ACM} Press},
  year      = {2022}
}

@inproceedings{Qu2022PPSN,
  author    = {Benjamin Doerr and Zhongdi Qu},
  title     = {A First Runtime Analysis of the {NSGA-II} on a Multimodal Problem},
  booktitle = {Proceedings of the International Conference on Parallel Problem Solving from Nature},
  series = {PPSN XVII},
  novolume    = {13399},
  pages     = {399--412},
  publisher = {Springer},
  year      = {2022}
}

@article{Deb2002,
  author    = {Kalyanmoy Deb and
               Amrit Pratap and
               Sameer Agarwal and
               T. Meyarivan},
  title     = {A Fast and Elitist Multiobjective Genetic Algorithm: {NSGA-II}},
  journal   = {{IEEE} Transactions on Evolutionary Computation},
  volume    = {6},
  number    = {2},
  pages     = {182-197},
  year      = {2002}
}

@misc{NSGAIICode2011,
  author    = {Kalyanmoy Deb},
  title     = {{NSGA-II} Source Code in {C}, version 1.1.6},
  howpublished = {\url{https://www.egr.msu.edu/~kdeb/codes/nsga2/nsga2-gnuplot-v1.1.6.tar.gz}},
  year      = {2011},
  note      = {Accessed: 2022-08-15}
}

@book{Jansen2013,
  author    = {Thomas Jansen},
  title     = {Analyzing Evolutionary Algorithms -- The Computer Science Perspective},
  series    = {Natural Computing Series},
  publisher = {Springer},
  year      = {2013},
  noisbn      = {978-3-642-17338-7}
}

@article{SUDHOLT20092511,
title = {The impact of parametrization in memetic evolutionary algorithms},
journal = {Theoretical Computer Science},
volume = {410},
number = {26},
pages = {2511-2528},
year = {2009},
noissn = {0304-3975},
nodoi = {https://doi.org/10.1016/j.tcs.2009.03.003},
nourl = {https://www.sciencedirect.com/science/article/pii/S0304397509002072},
author = {Dirk Sudholt}
}

@inproceedings{Opris2025PAES,
author = {Opris, Andre},
title = {A First Runtime Analysis of the {PAES-25}: An Enhanced Variant of the {Pareto} Archived Evolution Strategy},
year = {2025},
publisher = {ACM},
nourl = {https://doi.org/10.1145/3729878.3746618},
nodoi = {10.1145/3729878.3746618},
booktitle = {Proceedings of the Foundations of Genetic Algorithms},
pages = {202–213},
numpages = {12},
series = {FOGA 2025}
}

@inproceedings{MOEASubset,
author = {Deng, Renzhong and Zheng, Weijie and Li, Mingfeng and Liu, Jie and Doerr, Benjamin},
title = {Runtime Analysis for State-of-the-Art Multi-objective Evolutionary Algorithms on the Subset Selection Problem},
year = {2024},
noisbn = {978-3-031-70070-5},
publisher = {Springer},
booktitle = {Proceedings of the International Conference on Parallel Problem Solving from Nature},
series = {PPSN XVIII},
year = {2024},
noaddress = {Berlin, Heidelberg},
nourl = {https://doi.org/10.1007/978-3-031-70071-2_17},
nodoi = {10.1007/978-3-031-70071-2_17},
pages = {264–279},
nolocation = {Hagenberg, Austria}
}

@article{DangOprisSudholt2026Lessons,
author = {Dang, Duc-Cuong and Opris, Adrian and Sudholt, Dirk},
title = {Why Dominance is Not Enough: Lessons from Practical Evolutionary Multi-objective Algorithms},
journal = {Algorithmica},
volume = {88},
pages = {60},
year = {2026}
}

@inproceedings{Krejca2025b,
  author    = {Benjamin Doerr and Tudor Ivan and Martin S. Krejca},
  title     = {Speeding Up the {NSGA-II} With a Simple Tie-Breaking Rule},
  booktitle = {Proceedings of the {AAAI} Conference on Artificial Intelligence},
  series = {AAAI~2025},
  pages     = {26964-26972},
  publisher = {{AAAI} Press},
  year      = {2025},
  nodoi       = {10.1109/CEC.2013.6557784}
}

@inproceedings{Zheng_Doerr_2024,
author={Zheng, Weijie and Doerr, Benjamin},
title={Runtime Analysis of the {SMS-EMOA} for Many-Objective Optimization}, 
novolume={38}, 
nourl={https://ojs.aaai.org/index.php/AAAI/article/view/30077}, 
nodoi={10.1609/aaai.v38i18.30077}, 
nonumber={18}, 
booktitle={Proceedings of the AAAI Conference on Artificial Intelligence}, 
series = {AAAI 2024},
publisher={{AAAI} Press},
year={2024}, 
pages={20874-20882}}

@inproceedings{doerr2025tightruntimeGSEMO,
  author    = {Benjamin Doerr and Martin S. Krejca and Andre Opris},
  title     = {Tight Runtime Guarantees From Understanding the Population Dynamics of the {GSEMO} Multi-Objective Evolutionary Algorithm},
  booktitle = {Proceedings of the Thirty-Fourth International Joint Conference on Artificial Intelligence},
  series = {IJCAI 2025},
  pages     = {8876--8884},
  publisher = {ijcai.org},
  noaddress   = {Montreal, Canada},
  year      = {2025},
  nourl     = {https://doi.org/10.24963/ijcai.2025/987}
}

@inproceedings{DoerrNearTight,
author = {Wietheger, Simon and Doerr, Benjamin},
title = {Near-Tight Runtime Guarantees for Many-Objective Evolutionary Algorithms},
year = {2024},
pages = {153--168},
noisbn = {978-3-031-70084-2},
nourl = {https://doi.org/10.1007/978-3-031-70085-9_10},
nodoi = {10.1007/978-3-031-70085-9_10},
booktitle = {Proceedings of the International Conference on Parallel Problem Solving from Nature},
series = {PPSN XVIII}
}

@inproceedings{DoerrQ23b,  OPTnote={no JV},
  author       = {Benjamin Doerr and Zhongdi Qu},
  OPTeditor       = {Brian Williams and Yiling Chen and Jennifer Neville},
  title        = {Runtime Analysis for the {NSGA-II:} Provable Speed-ups from Crossover},
  booktitle = {Proceedings of the {AAAI} Conference on Artificial Intelligence},
  series = {{AAAI} 2023},
  pages        = {12399--12407},
  publisher    = {{AAAI} Press},
  year         = {2023}
  }

@inproceedings{ijcai2025p988,
  title     = {Scalable Speed-ups for the {SMS-EMOA} from a Simple Aging Strategy},
  author    = {Li, Mingfeng and Zheng, Weijie and Doerr, Benjamin},
  booktitle = {Proceedings of the Thirty-Fourth International Joint Conference on
               Artificial Intelligence},
  series = {IJCAI 2025},
  publisher = {ijcai.org},
  pages     = {8885--8893},
  year      = {2025}
}

@inproceedings{Opris26PopDyn,
      title={Towards a Rigorous Understanding of the Population Dynamics of the {NSGA-III}: Tight Runtime Bounds}, 
      author={Andre Opris},
      year={2026},
      pages={37125-37133},
      booktitle={Proceedings of the AAAI Conference on Artificial Intelligence},
      series = {AAAI 2026},
      publisher={{AAAI} Press},
      noaddress = {Singapore}
}

@inproceedings{DoerrKS2026,
  author       = {Benjamin Doerr and
                  Martin S. Krejca and
                  Milan Stankovic},
  noeditor       = {Sven Koenig and
                  Chad Jenkins and
                  Matthew E. Taylor},
  title        = {Improved Runtime Guarantees for the {SPEA2} Multi-Objective Optimizer},
  booktitle    = {Proceedings of the {AAAI} Conference on Artificial Intelligence},
  series       = {AAAI 2026},
  pages        = {36855--36863},
  publisher    = {{AAAI} Press},
  year         = {2026},
  nourl          = {https://doi.org/10.1609/aaai.v40i43.41012},
  nodoi          = {10.1609/AAAI.V40I43.41012}
}

@INPROCEEDINGS{NSGASTEADYSTATE,
  author={Mishra, Sumit and Mondal, Samrat and Saha, Sriparna},
  booktitle={IEEE Congress on Evolutionary Computation}, 
  series = {CEC 2016},
  title={Fast implementation of steady-state {NSGA-II}}, 
  year={2016},
  volume={},
  number={},
  pages={3777-3784},
  nokeywords={Steady-state;Sociology;Statistics;Sorting;Time complexity;Evolutionary computation},
  nodoi={10.1109/CEC.2016.7744268}}

@inproceedings{WiethegerD23,
  author       = {Simon Wietheger and
                  Benjamin Doerr},
  title        = {A Mathematical Runtime Analysis of the Non-dominated Sorting Genetic
                  Algorithm {III} {(NSGA-III)}},
  booktitle    = {Proceedings of the International Joint Conference on Artificial Intelligence},
  series = {IJCAI 2023},
  pages        = {5657--5665},
  nopublisher    = {ijcai.org},
  year         = {2023},
  nourl          = {https://doi.org/10.24963/ijcai.2023/628},
  nodoi          = {10.24963/IJCAI.2023/628},
  bibsource    = {dblp computer science bibliography, https://dblp.org},
  publisher = {ijcai.org}
}

@article{Cerf2023,
  author       = {Sacha Cerf and
                  Benjamin Doerr and
                  Benjamin Hebras and
                  Yakob Kahane and
                  Simon Wietheger},
  title        = {The First Proven Performance Guarantees for the Non-Dominated Sorting
                  Genetic Algorithm {II} {(NSGA-II)} on a Combinatorial Optimization
                  Problem},
  nojournal      = {CoRR},
  volume       = {abs/2305.13459},
  year         = {2023},
  nourl          = {https://doi.org/10.48550/arXiv.2305.13459},
  nodoi          = {10.48550/arXiv.2305.13459},
  eprinttype    = {arXiv},
  eprint       = {2305.13459},
  bibsource    = {dblp computer science bibliography, https://dblp.org}
}

@inproceedings{RenBLQ24,
  author       = {Shengjie Ren and Chao Bian and
                  Miqing Li and Chao Qian},
  OPTeditor       = {Michael Affenzeller and
                  Stephan M. Winkler and
                  Anna V. Kononova and
                  Heike Trautmann and
                  Tea Tusar and
                  Penousal Machado and
                  Thomas B{\"{a}}ck},
  title        = {A first running time analysis of the {S}trength {P}areto {E}volutionary
                  {A}lgorithm~2 {(SPEA2)}},
  booktitle    = {Proceedings of the International Conference of Parallel Problem Solving from Nature},
  series = {{PPSN} XVIII},
  pages        = {295--312},
  publisher    = {Springer},
  year         = {2024}
}

@inproceedings{OprisNSGAIII,
author = {Opris, Andre and Dang, Duc-Cuong and Neumann, Frank and Sudholt, Dirk},
title = {Runtime Analyses of {NSGA-III} on Many-Objective Problems},
year = {2024},
publisher = {{ACM} Press},
nourl = {https://doi.org/10.1145/3638529.3654218},
nodoi = {10.1145/3638529.3654218},
booktitle = {Proceedings of the Genetic and Evolutionary Computation Conference},
series = {GECCO 2024},
pages = {1596–1604},
numpages = {9}
}

@inproceedings{Opris2025,
  title={A Many Objective Problem Where Crossover is Provably Indispensable},
  author={Andre Opris},
  year={2025},
  series = {{AAAI} 2025},
  publisher = {AAAI Press},
  BOOKTITLE = {Proceedings of the AAAI Conference on Artificial Intelligence},
  pages = {27108-27116},
  nourl = {https://doi.org/10.1609/aaai.v39i25.34918},
}

@article{DangEfficient2024,
author = {Dang, Duc-Cuong and Opris, Andre and Sudholt, Dirk},
title = {Runtime Analysis of Functions where Widely Used Evolutionary Multi-Objective Algorithms Beat Simple Ones},
journal = {ACM Transactions on Evolutionary Learning and Optimization},
volume = {6},
number = {2},
pages = {1--33},
year = {2025}
}

@article{Dang2024,
    title = {Crossover can guarantee exponential speed-ups in evolutionary multi-objective optimisation},
    journal = {Artificial Intelligence},
    volume = {330},
    pages = {104098},
    year = {2024},
    noissn = {0004-3702},
    nodoi = {https://doi.org/10.1016/j.artint.2024.104098},
    nourl = {https://www.sciencedirect.com/science/article/pii/S0004370224000341},
    author = {Duc-Cuong Dang and Andre Opris and Dirk Sudholt}
}

\end{document}